\documentclass{article}

\usepackage{microtype}
\usepackage{graphicx}
\usepackage{subcaption}
\usepackage{booktabs} % for professional tables

\usepackage{hyperref}

\makeatletter
\expandafter\def\csname ver@algorithmic.sty\endcsname{1000/00/00}
\makeatother

\usepackage[accepted]{icml2026}

\usepackage{amsmath}
\usepackage{amssymb}
\usepackage{mathtools}
\usepackage{amsthm}

\usepackage{booktabs}      
\usepackage{multirow}      
\usepackage[table]{xcolor}
\newcommand{\good}[1]{\cellcolor{green!15}#1} 
\newcommand{\bad}[1]{\cellcolor{red!15}#1}

\DeclareMathOperator{\Prob}{\mathbb{P}}

\usepackage{amsmath,amssymb,amsthm}
\usepackage{graphicx}
\usepackage{bm}  
\usepackage{algorithm}
\usepackage[noend]{algpseudocode}
\usepackage{enumitem}

\newcommand{\blfootnote}[1]{%
  \begingroup
  \renewcommand\thefootnote{}\footnote{#1}%
  \addtocounter{footnote}{-1}%
  \endgroup
}

\usepackage{makecell}
\usepackage{xcolor}
\usepackage{listings}
\usepackage{subcaption}

\lstdefinelanguage{json}{
  basicstyle=\ttfamily\small,
  stepnumber=1,
  showstringspaces=false,
  breaklines=true,
  literate=
   *{0}{{{\color{numb}0}}}{1}
    {1}{{{\color{numb}1}}}{1}
    {2}{{{\color{numb}2}}}{1}
    {3}{{{\color{numb}3}}}{1}
    {4}{{{\color{numb}4}}}{1}
    {5}{{{\color{numb}5}}}{1}
    {6}{{{\color{numb}6}}}{1}
    {7}{{{\color{numb}7}}}{1}
    {8}{{{\color{numb}8}}}{1}
    {9}{{{\color{numb}9}}}{1}
    {:}{{{\color{punct}{:}}}}{1}
    {,}{{{\color{punct}{,}}}}{1}
    {\{}{{{\color{delim}{\{}}}}{1}
    {\}}{{{\color{delim}{\}}}}}{1}
    {[}{{{\color{delim}{[}}}}{1}
    {]}{{{\color{delim}{]}}}}{1},
}

\definecolor{numb}{RGB}{188,75,184}
\definecolor{punct}{RGB}{66,66,66}
\definecolor{delim}{RGB}{45,156,219}

\usepackage{tikz}
\usetikzlibrary{positioning, shapes.geometric, arrows.meta, positioning, shadows, calc, fit, backgrounds}

\usepackage{stfloats}
\usepackage{xurl} % Breaks long URLs at any character
\usepackage[capitalize,noabbrev]{cleveref}

\theoremstyle{plain}
\newtheorem{theorem}{Theorem}[section]
\newtheorem{proposition}[theorem]{Proposition}
\newtheorem{lemma}[theorem]{Lemma}

\theoremstyle{definition}

\newtheorem{assumption}[theorem]{Assumption}
\theoremstyle{remark}

\usepackage[disable,textsize=tiny]{todonotes}
\icmltitlerunning{Counterfactual Residual Data Augmentation for
Regression}

\begin{document}

\twocolumn[
  \icmltitle{Counterfactual Residual Data Augmentation for
Regression}

  % It is OKAY to include author information, even for blind submissions: the
  % style file will automatically remove it for you unless you've provided
  % the [accepted] option to the icml2026 package.

  % List of affiliations: The first argument should be a (short) identifier you
  % will use later to specify author affiliations Academic affiliations
  % should list Department, University, City, Region, Country Industry
  % affiliations should list Company, City, Region, Country

  % You can specify symbols, otherwise they are numbered in order. Ideally, you
  % should not use this facility. Affiliations will be numbered in order of
  % appearance and this is the preferred way.
  \icmlsetsymbol{equal}{*}

  \begin{icmlauthorlist}
    \icmlauthor{Hossein Mohebbi}{uw,vector}
    \icmlauthor{Oliver Schulte}{sfu}
    \icmlauthor{Ke Li}{sfu}
    \icmlauthor{Pascal Poupart}{uw,vector}
  \end{icmlauthorlist}

  \icmlaffiliation{uw}{University of Waterloo}
  \icmlaffiliation{vector}{Vector Institute}
  \icmlaffiliation{sfu}{Simon Fraser University}

  \icmlcorrespondingauthor{Hossein Mohebbi}{hossein.mohebbi@uwaterloo.ca}

  % You may provide any keywords that you find helpful for describing your
  % paper; these are used to populate the "keywords" metadata in the PDF but
  % will not be shown in the document
  \icmlkeywords{Machine Learning, ICML, Counterfactual Reasoning, Data Augmentation, Regression, Causality, Residuals}

  \vskip 0.3in
]

% this must go after the closing bracket ] following \twocolumn[ ...

% This command actually creates the footnote in the first column listing the
% affiliations and the copyright notice. The command takes one argument, which
% is text to display at the start of the footnote. The \icmlEqualContribution
% command is standard text for equal contribution. Remove it (just {}) if you
% do not need this facility.

% Use ONE of the following lines. DO NOT remove the command.
% If you have no special notice, KEEP empty braces:
\printAffiliationsAndNotice{}  % no special notice (required even if empty)
% Or, if applicable, use the standard equal contribution text:
% \printAffiliationsAndNotice{\icmlEqualContribution}

\begin{abstract}
Data-driven modeling in real-world regression tasks often suffers from limited training samples, high collection costs, and noisy observations. Inspired by the impact of data augmentation in vision and language, 
we propose a novel \emph{Counterfactual Residual Data Augmentation (CRDA)} technique for tabular regression. Our key insight is that once a regressor has modeled the systematic component of the data, the remaining noise can be viewed as an invariant residual that remains stable under small perturbations of carefully selected features. We exploit this residual invariance to generate new, yet realistic, training samples, effectively expanding the dataset without requiring additional real data. Our method is model-agnostic and readily applicable to various types of regressors. In experiments across datasets from a variety of benchmark repositories, on average, \emph{CRDA} reduces an \textit{MLP Regressor}'s MSE by \textbf{22.9\%} and an \textit{XGBoost Regressor}'s MSE by \textbf{6.4\%}. When compared to existing state-of-the-art data generators and augmentation techniques, CRDA consistently outperforms in MSE reduction. By adding principled counterfactual variations to the training data, our method offers a simple and efficient remedy for noise-prone, small-sample regression settings.
\end{abstract}

\section{Introduction}
\label{sec:intro}

Data scarcity and noise are frequently encountered obstacles in regression tasks across domains such as medicine, finance, and manufacturing. Collecting large-scale, high-quality data can be expensive or impractical, and existing data augmentation techniques, while well developed in computer vision and NLP, often do not translate naturally to tabular regression. As a result, many supervised learning models fail to fully capture the underlying behavior of real-world processes when only limited training examples are available. 

In this paper, we propose \textbf{Counterfactual Residual Data Augmentation (CRDA)}, a simple and flexible method to bolster regression performance under small data constraints. The core idea is straightforward: 
\emph{(i)} we train a base predictor (e.g., MLP or XGBoost) on a dataset, 
\emph{(ii)} identify one or more features whose perturbations do not alter the residual (prediction error) distribution significantly, and 
\emph{(iii)} generate new samples by modifying those features while preserving the original ``noise'' or residual component. 
To illustrate, consider a house price prediction task. A model captures systematic value drivers like location and square footage, while the residual captures unobserved factors like a \textit{bidding war} driven by a specific buyer's urgency. Our key insight is that varying a secondary feature, such as garage finish, changes the systematic price but is unlikely to alter the specific buyer's urgency. 
CRDA exploits this independence to synthesize a residual-preserving counterfactual: a house with a different garage finish, an updated systematic price, but the exact same ``bidding war'' residual.

This pattern recurs across domains: in clinical recovery-time prediction, perturbing the amount of physical therapy changes a patient's recovery time but does not alter their unobserved physiology; in crop-yield regression, perturbing fertilizer dose shifts the predicted yield while leaving the field's unmeasured pest pressure untouched; in home energy forecasting, perturbing the thermostat setpoint moves predicted consumption but not whether a window was left open that day. 
\blfootnote{Project page with reproducible experiment code and Python package: \url{https://crda-project.github.io/}.}

\textbf{Motivation and Benefits.}
Our motivation stems from the difficulty of acquiring sufficient labeled data in many practical applications, coupled with the risk of overfitting when sample sizes are small. A major attraction of \emph{CRDA} is its ability to insert new data points \emph{without} assuming domain-specific transformations or heuristics. Instead, it relies on a learned predictor to separate systematic behavior from noise, then conserves the latter across minor interventions of selected features. As a result, the augmented samples remain consistent with the underlying distributional assumptions, improving model fit and reducing variance. Empirically, we observe double-digit percent reductions in test error for small-sample regression tasks, demonstrating the utility of our approach across a variety of dataset types.

Our work makes the following key contributions:

\textbf{(1) New Data Augmentation Framework.} We introduce a model-agnostic strategy for augmenting tabular regression data, centered on counterfactual reasoning and residual invariance.

\textbf{(2) Residual Invariance Principle.} We formalize how certain features can be perturbed without corrupting the noise structure, providing insights to guide feature selection.

\textbf{(3) Empirical Validation.} We evaluate CRDA on synthetic benchmarks and real-world datasets from standard repositories (e.g., UCI, PMLB), illustrating consistent improvements across neural and ensemble models.

\section{Related Work}
\label{sec:related}

\textbf{General Data Augmentation.}
Data augmentation refers to the strategy of enlarging or diversifying a training set via synthetic transformations. While central to success in computer vision and NLP~\citep{zhang2018mixup,yun2019cutmix,cubuk2019autoaugment}, these techniques often rely on \emph{label-preserving} symmetries (e.g.\ image rotations) or domain-specific invariances (e.g.\ back-translation in text). However, applying these methods to tabular regression remains non-trivial.

\textbf{Tabular and Regression-Specific Augmentation.}
Classical oversampling approaches include SMOTE \citep{chawla2002smote} and its regression extensions \citep{branco2017smogn}, which interpolate between samples but do not necessarily preserve higher-order feature interactions or heteroskedastic noise. 
Recent advances seek to formalize regression augmentation through geometric properties. 
For example, RegMix \citep{hwang2021regmix} optimizes Mixup policies to generate samples within high-density regions of the data manifold, aiming to preserve the underlying structure. Conversely, C-Mixup \citep{yao2022cmixup} addresses the risk of manifold intrusion by restricting mixing to sample pairs with high label similarity.
Closely related is Anchor Data Augmentation (ADA) \citep{schneider2023anchor}, which extends Anchor Regression to augmentation. ADA identifies ``anchors" (by clustering) and generates samples using a first-order Taylor approximation, effectively assuming local linearity within clusters.
While these methods enforce geometric regularity (linearity or manifold density), they can struggle in highly non-linear or sparse regimes where local linearity assumptions fail. CRDA aims to avoid this by enforcing \textit{statistical} regularity (residual invariance) instead.

\textbf{Deep Generative Models.}
Deep generative models offer an alternative by learning the joint distribution to sample entirely new rows. Approaches like CTGAN \citep{xu2019ctgan}, TVAE \citep{xu2019ctgan}, and TabDDPM \citep{kotelnikov2023tabddpm} have shown promise in privacy-preserving data synthesis. 
However, these models typically treat the target variable as just another column, failing to preserve an instance’s specific residual noise. This often leads to ``realistic" looking samples that degrade predictive performance.

\textbf{Residual Bootstrapping.}
In statistical literature, residuals have been leveraged extensively for \emph{uncertainty quantification} rather than data augmentation. For example, the residual bootstrap \citep{efron1992bootstrap} and conformal prediction methods \citep{barber2021jackknifeplus} resample or reuse residuals to construct confidence intervals. Our work repurposes this mechanism for augmentation.

\textbf{Causal and Counterfactual Data Augmentation.}
Data augmentation typically ignores the generative process behind the data, risking unrealistic synthetic examples. Causal-based approaches \citep{kocaoglu2018causalgan,arjovsky2019invariant} propose integrating structural assumptions so that augmentations preserve invariant relationships across environments. This has been explored in computer vision through interventions on object attributes, and in language by editing tokens in a \texttt{do}-intervention style.
Closely related at the framing level, concurrent work by \citet{akbar2025an} formalizes outcome-invariant augmentation as a soft intervention on the treatment-generating mechanism, requiring that the augmentation preserves the outcome function itself ($f(gx) = f(x)$). CRDA takes a complementary route where instead of assuming a known symmetry of the predictor, it preserves the \emph{instance-specific noise term} under perturbations of selected features.
This noise-preservation principle draws a direct parallel to work in reinforcement learning. \citet{lu2020sample} showed that next-state samples remain identifiable under mild assumptions (monotonicity and independence in the noise term). Their \textit{Theorem~1} establishes that, once the observed outcome fixes a particular noise quantile, reusing that noise in a “what-if” scenario yields a valid counterfactual next-state. 
Similarly, CRDA treats the calculated residual as an exogenous noise variable that is assumed to be independent of the features being perturbed. This allows us to produce new data in a more systematic and grounded way than purely generative or interpolation-based techniques.

\section{Background}
\label{sec:background}

In this section, we provide an overview of the key concepts that motivate our proposed approach.

\subsection{Counterfactual Reasoning and Structural Causal Models}
\label{sec:background_causal}

A \emph{structural causal model} (SCM) \citep{pearl2009causality, peters2017elements} formalizes how observed variables are generated by underlying data-generating processes (DGP). Formally, an SCM is specified by a tuple \(\bigl(\mathcal{X}, \mathcal{Z}, F, P_{\mathcal{Z}}\bigr)\), where:
\begin{itemize}
    \item \(\mathcal{X} = \{X_1, \ldots, X_m\}\) is the set of endogenous (observed) variables,
    \item \(\mathcal{Z} = \{Z_1, \ldots, Z_m\}\) is the set of exogenous (noise) variables with distribution \(P_{\mathcal{Z}}\),
    \item \(F = \{f_i\}_{i=1}^m\) is a collection of structural equations, each of the form
    \[
        x_i \;\leftarrow\; f_i\bigl(pa_i,z_i) 
        %= g(\mathcal{X}-\{X_i\}) + Z_i,
    \]
    where \(\mathrm{Pa}_i\subseteq \mathcal{X}\setminus\{X_i\}\) denotes the parents of \(X_i\).
\end{itemize}

An SCM induces a graph, which encodes causal relationships (i.e.\ who influences whom), and the exogenous noise terms capture stochasticity.

\textbf{Interventions and causal effects.}
A central notion in causal inference is that of an \emph{intervention}, written \(\operatorname{do}(\cdot)\) \citep{pearl2009causality}. By applying \(\operatorname{do}(X = x')\), one replaces the original structural equation \(X \leftarrow f_X(\mathrm{Pa}, Z)\) with a constant assignment \(X \leftarrow x'\). This operation severs incoming edges to \(X\), thus altering the downstream (child) variables but leaving other aspects of the system intact. Interventions enable us to reason about population-level outcomes under a forced assignment.

\textbf{Counterfactuals.}
While interventions ask about population-level effects,
counterfactual reasoning answers instance-level ``what if'' questions: \emph{given this specific outcome I observed, what would have happened if some feature had been different?} \citep{pearl2009causality,peters2017elements}. Concretely, one first infers the actual setting of \(Z = z\), then imagines how the outcome \(Y\) would change under a hypothetical intervention \(\operatorname{do}(X = x')\). This process involves:

\emph{(i)} \textbf{abduction}, where we infer \(\mathbf{z}\) from the observed data,

\emph{(ii)} \textbf{action}, where we intervene and reassign \(X\),

\emph{(iii)} \textbf{prediction}, where we propagate \(z\) through the modified system to obtain the counterfactual outcome distribution $P(Y'|X = x',Z = z)$.

\subsection{Main Assumptions and Theory}
\label{sec:assumptions}

Throughout the paper, we assume that the data-generating process follows an \textbf{additive-noise structural causal model (SCM)}, which decomposes the outcome $Y$ into a systematic component and a residual component. The systematic component is the conditional expectation estimated by a base regressor $g$: 

\begin{equation}
    Y = g(X_P, X_R) + Z %; Z \mbox{ independent of } {X_P|X_R}
\end{equation}

The core theoretical principle underpinning CRDA is the assumption of \emph{residual invariance}. This principle posits that for the true conditional-expectation regressor, the residual noise term remains distributionally constant under interventions on a specific subset of features. We formalize this as follows.

\begin{assumption} \label{assump:cond}
Let the feature vector $X$ be partitioned into two disjoint subsets, $X = (X_P, X_R)$, where $X_P$ are the features we intend to perturb (the \emph{perturbable} coordinates) and $X_R$ are the features we hold fixed (the \emph{remaining} coordinates). Let $g(X) = \mathbb{E}[Y|X]$ be the true conditional expectation function, and let $Z = Y - g(X)$ be the corresponding structural noise term. Residual invariance requires that the noise $Z$ is conditionally independent of the perturbable features $X_P$ given the fixed features $X_R$:
\begin{align}
    \Prob(Z \mid X_P, X_R) = \Prob(Z \mid X_R) \label{eq:indep}
\end{align}
%Equation \ref{eq:indep} says  
\end{assumption}

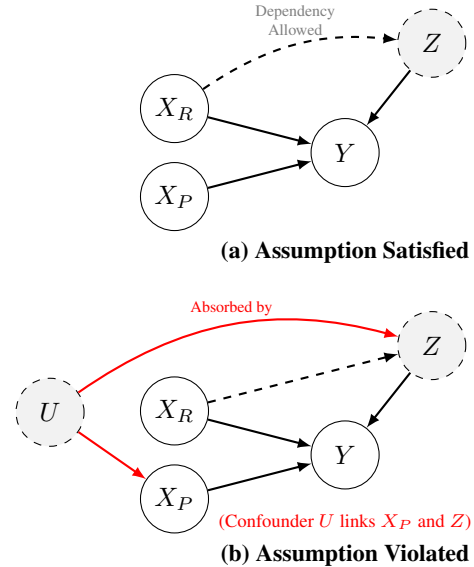
\begin{figure}[h]
\centering
\begin{tikzpicture}[
    node distance=1.5cm and 1.5cm,
    observed/.style={circle, draw, minimum size=0.9cm, inner sep=0pt, fill=white},
    latent/.style={circle, draw, dashed, minimum size=0.9cm, inner sep=0pt, fill=gray!10},
    arrow/.style={->, >=latex, thick},
    label_text/.style={font=\small\bfseries}
]

% ---------------------------------------------------------
% PANEL A: ASSUMPTION SATISFIED
% ---------------------------------------------------------
\begin{scope}[local bounding box=GraphA]
    % Nodes
    \node[observed] (XR) {$X_R$};
    \node[observed, below=0.25cm of XR] (XP) {$X_P$};
    \node[observed, right=1.8cm of $(XR)!0.5!(XP)$] (Y) {$Y$};
    \node[latent, above right=0.8cm and 0.5cm of Y] (Z) {$Z$};

    % Edges
    \draw[arrow] (XR) -- (Y);
    \draw[arrow] (XP) -- (Y);
    \draw[arrow] (Z) -- (Y);
    
    % The "Allowed" dependency: XR -> Z
    \draw[arrow, dashed, bend left=20] (XR) to node[midway, above, font=\tiny, text=gray, align=center] {Dependency\\Allowed} (Z);
    
    % Label
    \node[below=0.6cm of Y, label_text] {(a) Assumption Satisfied};
\end{scope}

% ---------------------------------------------------------
% PANEL B: ASSUMPTION VIOLATED
% ---------------------------------------------------------
% CHANGED: Swapped xshift for yshift to move it down instead of right
\begin{scope}[yshift=-4cm, local bounding box=GraphB]
    % Nodes (Same positions relative to this scope)
    \node[observed] (XR2) {$X_R$};
    \node[observed, below=0.25cm of XR2] (XP2) {$X_P$};
    \node[observed, right=1.8cm of $(XR2)!0.5!(XP2)$] (Y2) {$Y$};
    \node[latent, above right=0.8cm and 0.5cm of Y2] (Z2) {$Z$};
    
    % The Unobserved Confounder
    \node[latent, above left=0.5cm and 1.0cm of XP2] (U) {$U$};

    % Standard Edges
    \draw[arrow] (XR2) -- (Y2);
    \draw[arrow] (XP2) -- (Y2);
    \draw[arrow] (Z2) -- (Y2);
    % Allowed part
    \draw[arrow, dashed, bend left=20] (XR2) -- (Z2); 

    % The Violation Edges
    \draw[arrow, red, thick] (U) -- (XP2);
    \draw[arrow, red, thick, bend left=25] (U) to node[midway, above, font=\tiny, text=red] {Absorbed by} (Z2);

    % Label
    \node[below=0.6cm of Y2, label_text] {(b) Assumption Violated};
    \node[below=0.2cm of Y2, font=\scriptsize, align=center, text=red] {(Confounder $U$ links $X_P$ and $Z$)};
\end{scope}

\end{tikzpicture}
\caption{\textbf{Causal Visualization of Residual Invariance.} 
\textbf{(a)} The structure satisfies Assumption~\ref{assump:cond}.
\textbf{(b)} A violation due to unobserved confounding. Here, a latent variable $U$ causes both $X_P$ and $Y$. Since the residual $Z$ absorbs the variation of $U$, a dependency is created between $X_P$ and $Z$, invalidating the augmentation.}
\label{fig:causal_graphs}
\end{figure}

\textbf{Causal Interpretation and Hidden Confounding.}
To better understand what underlying DGPs meet Assumption~\ref{assump:cond}, Figure~\ref{fig:causal_graphs}a visualizes a causal structure that satisfies it. 
The assumption requires that the features selected for perturbation, $X_P$, are \textit{exogenous} with respect to the residual mechanism $Z$. 
Note that $Z$ is allowed to depend on the fixed features $X_R$ (e.g., heteroscedasticity or confounding on $X_R$), as long as it remains independent of $X_P$. 
Appendix~\ref{app:causal} shows that the causal (d-separation) analogue of Assumption~\ref{assump:cond} implies the following:  %property of perturbed features: 
if $Z$ is the only latent direct cause of $Y$, a perturbed feature $X \in X_P$ is either itself a direct cause of $Y$  or is independent of (d-separated from) $Y$ given the observed direct causes. Thus Assumption~\ref{assump:cond} ensures that the only perturbed features used to predict $Y$ are direct causes of $Y$ (cf. Figure~\ref{fig:causal_graphs}a).

Recent work highlights that unobserved confounding is a primary driver of distribution shift failures in real-world tabular data~\citep{prashant2025scalable, reddy2026when}.
When a latent confounder $U$ exists, its influence on $Y$ that is not explained by $X$ is absorbed into the residual term $Z$. Consequently, $Z$ acts as a noisy proxy for $U$. Figure~\ref{fig:causal_graphs}b illustrates the case where $U$ causes both $X_P$ and $Y$, thereby creating a ``backdoor path'' $X_P \leftarrow U \to Y$. Here, a statistical dependence arises between $X_P$ and $Z$, violating Assumption~\ref{assump:cond}. 
Therefore, the validity of counterfactual augmentation depends on identifying and excluding such confounded features from the perturbation set.

\subsection{From Population to Empirical Residuals}
\label{sec:empirical_residuals}
The discussion so far has been at the population level. We now consider what happens when CRDA operates on finite-sample estimates of $g$ and $Z$.
Assumption~\ref{assump:cond} is stated in terms of the true conditional expectation $g(X) = \mathbb{E}[Y \mid X]$ and the true structural noise $Z = Y - g(X)$. In practice, CRDA does not have access to $g$; it uses a finite-data estimate $\hat{g}$ and computes \emph{empirical} residuals $\hat{Z} = Y - \hat{g}(X)$. The relationship between empirical and true residuals is
\[
    \hat{Z} \;=\; Z + \bigl(g(X) - \hat{g}(X)\bigr) \;=\; Z + e(X_P, X_R),
\]
where $e(X_P, X_R)$ is the pointwise approximation error of the base model. The validity of CRDA on empirical residuals depends on how this error behaves on the perturbable coordinates.

\textbf{Well-specified regime.} When $\hat{g}$ is sufficiently accurate, $e \to 0$ uniformly and $\hat{Z} \to Z$. Assumption~\ref{assump:cond} transfers to the empirical residuals: $\hat{Z} \perp X_P \mid X_R$.

\textbf{Misspecified regime.} When $\hat{g}$ fails to capture the systematic effect of $X_P$, either because the base model underfits or because $X_P$ has a complex non-linear contribution that requires more data than is available, $e$ retains systematic dependence on $X_P$ even after conditioning on $X_R$. Then $\hat{Z} \not\perp X_P \mid X_R$, and Assumption~\ref{assump:cond} fails on the empirical residuals.

This split governs when CRDA's residual-reuse construction is justified at the empirical level; the practical filtering and validation mechanisms that target the misspecified regime are introduced in Section~\ref{sec:method}.

\section{Method}
\label{sec:method}

In this section, we detail our proposed \emph{Counterfactual Residual Data Augmentation (CRDA)} procedure for tabular regression. The main goal is to augment a limited dataset by synthesizing new samples that remain true to the original noise distribution.
Algorithm~\ref{alg:crda} outlines the full workflow.
For a visual depiction of the pipeline and how the residual and systematic components interact, see Figure~\ref{fig:crda_pipeline} in Appendix~\ref{app:pipeline_diagram}.

\subsection{Algorithmic Overview}

\begin{algorithm}[h]
\caption{Counterfactual Residual Data Augmentation (CRDA)}
\label{alg:crda}
\begin{algorithmic}[1]
\Require Dataset $\mathcal{D}$, baseline regressor $\hat{g}(\cdot)$.
\Statex \quad \;\; Hyperparameters: \textsc{PerturbationRange} ($p$), \textsc{AugDataSizeFactor} ($M$).

\State \textbf{Split} $\mathcal{D}$ into $\mathcal{D}_{\mathrm{train}}$ and $\mathcal{D}_{\mathrm{test}}$.
\State \textbf{Train} baseline $\hat{g}(\cdot)$ on $\mathcal{D}_{\mathrm{train}}$.

\For{each $(\mathbf{x}_i, y_i)$ in $\mathcal{D}_{\mathrm{train}}$}
    \State $\hat{z}_i \gets y_i - \hat{g}(\mathbf{x}_i)$  \Comment{Compute residual}
\EndFor
% \State \textbf{Compute Residuals:} $z_i \gets y_i - g(\mathbf{x}_i)$ for all $i$.

\State \textbf{Select partition} $(X_P,X_R)$:
\begin{itemize}[leftmargin=1em]\setlength\itemsep{0em}
    \item PC algorithm to remove features directly connected to $\hat{Z}$.
    \item Correlation check to remove features strongly associated with $\hat{Z}$.
\end{itemize}
\If{$X_P = \emptyset$} 
    \Return $g$ \Comment{No partition found}
\EndIf

\State $\mathcal{D}_{\mathrm{aug}} \gets \emptyset$
\For{each $(\mathbf{x}_i, y_i) \in \mathcal{D}_{\mathrm{train}}$}
    \For{$m = 1$ to $M$}
        \State $\mathbf{x}'_{i,P} \gets \mathbf{x}_{i,P} \cdot (1 + \text{Unif}[-p, p])$
        \State $\mathbf{x}'_i \gets (\mathbf{x}'_{i,P}, \mathbf{x}_{i,R})$ \Comment{Keep $X_R$ fixed}
        \State $y'_i \gets \hat{g}(\mathbf{x}'_i) + \hat{z}_i$ \Comment{Preserve residual}
        \State $\mathcal{D}_{\mathrm{aug}} \gets \mathcal{D}_{\mathrm{aug}} \cup \{(\mathbf{x}'_i, y'_i)\}$
    \EndFor
\EndFor

\State \textbf{Validation:} Perform K-fold CV comparing \emph{unaugmented} vs.\ \emph{augmented} models.
\State Collect validation errors $\{e_{\mathrm{unaug}}^{(k)}, e_{\mathrm{aug}}^{(k)}\}_{k=1}^K$. 
\State $p\text{-value} \gets \text{WilcoxonSignedRank}\bigl(\{e_{\mathrm{unaug}}^{(k)}, e_{\mathrm{aug}}^{(k)}\}\bigr)$
\If{$p\text{-value} < \alpha$}
    \State \textbf{Train} new regressor $\hat{g}'$ on $\mathcal{D}_{\mathrm{train}} \cup \mathcal{D}_{\mathrm{aug}}$
    \State \Return $\hat{g}'$
\Else
    \State \Return $\hat{g}$ \Comment{Augmentation rejected}
\EndIf
\end{algorithmic}
\end{algorithm}

\textbf{Step 1: Data Splitting and Baseline Training.}
We begin by splitting the original dataset~$\mathcal{D}$ into training and test sets. Let 
\(
  \mathcal{D}_{\mathrm{train}} = \{(\mathbf{x}_i, y_i)\}_{i=1}^N
\)
and 
\(
  \mathcal{D}_{\mathrm{test}} = \{(\mathbf{x}_j, y_j)\}_{j=N+1}^{n}
\).
We then fit a \emph{baseline} regressor $\hat{g}(\cdot)$ on $\mathcal{D}_{\mathrm{train}}$. In practice, this model can be chosen from a variety of families (e.g.\ MLP, XGBoost) depending on the user’s preference.

\textbf{Step 2: Residual Computation.}
For each training sample $(\mathbf{x}_i,y_i)$, we compute the \emph{residual}
\[
   \hat{z}_i \;=\; y_i \;-\; \hat{g}(\mathbf{x}_i).
\]
Intuitively, $\hat{z}_i$ captures the latent factors not explained by $\hat{g}$.

\textbf{Step 3: Feature Selection.}
To satisfy Assumption~\ref{assump:cond}, we must identify which features belong to $X_P$ (perturbable) and which must remain in $X_R$ (fixed). Since the true causal graph is rarely known, we sequentially apply two filters to screen for (approximate) conditional independence between features and the residual $\hat{z}$:

\begin{enumerate}
    \item A \textbf{causal graph check} applies the Peter-Clark (PC) algorithm \citep{spirtes2000causation} to the joint set of variables $\{X, Y, \hat{Z}\}$. Features that have a direct edge to $\hat{Z}$ in the discovered graph are flagged as dependent and assigned to $X_R$. The remaining features are still potential candidates for $X_P$.
    \item A \textbf{Pearson correlation check} assigns features strongly correlated with $\hat{Z}$ (above a threshold) to $X_R$.
\end{enumerate}
The surviving coordinates form $X_P$; the complement forms $X_R$, satisfying Assumption \ref{assump:cond}. (If none survive, we skip augmentation and return the baseline $\hat{g}$.)

\textbf{Step 4: Counterfactual Generation.}
For each training sample $(\mathbf{x}_i, y_i)$, we generate $M$ synthetic samples (controlled by the hyperparameter \textsc{AugDataSizeFactor}). For each synthetic sample $m \in \{1, \dots, M\}$:
\begin{enumerate}
    \item We perturb the features in $X_P$ using a uniform scaling factor $\delta \sim \text{Unif}[-p, p]$, where $p\in(0,1)$ is the \textsc{PerturbationRange}:
    \[
        \mathbf{x}'_{i,P} = \mathbf{x}_{i,P} \cdot (1 + \delta).
    \]
    \item We construct the new feature vector $\mathbf{x}'_i = (\mathbf{x}'_{i,P}, \mathbf{x}_{i,R})$, keeping $X_R$ fixed.
    \item We calculate the counterfactual label by passing the \emph{new} input through the baseline model and adding the \emph{old} residual:
    \[
        y'_i = \hat{g}(\mathbf{x}'_i) + \hat{z}_i.
    \]
\end{enumerate}

Crucially, this preserves the original residual $\hat{z}_i$, thus keeping the overall noise structure intact under the perturbation. % (Proposition \ref{thrm:invariance}).
The newly generated samples $\bigl(\mathbf{x}_i',\,y_i'\bigr)$ form an augmented set \(\mathcal{D}_{\mathrm{aug}}\).

% \textbf{Step 6: Validation via Cross-Validation and Wilcoxon Signed-Rank Test.}
\textbf{Step 5: Safety Validation.}
Before committing to a final retraining, we evaluate whether the augmented samples \emph{significantly} improve generalization. Concretely, we run $K$-fold cross-validation on the \emph{original} training data, comparing two models:
\begin{enumerate}
    \item \textbf{Unaugmented model}: trained on the fold’s training portion as is.
    \item \textbf{Augmented model}: trained on the fold’s training portion \emph{plus} its augmented points (generated using the same procedure above).
\end{enumerate}
We collect the validation errors (e.g.\ MSE) across the $K$ folds for both models and perform a nonparametric \emph{Wilcoxon signed-rank test} \citep{wilcoxon1945individual} on the paired errors. If the resulting $p$-value is below a chosen significance level (e.g.\ $\alpha = 0.05$), we conclude that augmentation yields a statistically significant improvement; otherwise, we revert to the baseline $\hat{g}$.

\textbf{Step 6: Combined Dataset and Retraining.}
If the Wilcoxon test finds a significant improvement, we combine $\mathcal{D}_{\mathrm{train}}$ and $\mathcal{D}_{\mathrm{aug}}$ into a single augmented dataset
\[
   \mathcal{D}_{\mathrm{train}}^{\mathrm{aug}}
   \;=\;
   \mathcal{D}_{\mathrm{train}}
   \;\cup\;
   \mathcal{D}_{\mathrm{aug}}.
\]
We then retrain a \emph{final} model $\hat{g}'(\cdot)$ with this expanded dataset to use on the test set.

\subsection{Discussion of Key Design Choices}

\textbf{Two Views of CRDA: Statistical and Counterfactual.}

\emph{Statistical (non-causal) view.} For soundness, CRDA requires only Assumption~\ref{assump:cond}: $Z \perp X_P \mid X_R$. Given an estimated regressor $\hat{g}$, empirical residuals $\hat{z}_i = y_i - \hat{g}(x_i)$, a screened subset $X_P$, and a perturbed input $x'_i = (x'_{i,P}, x_{i,R})$, the augmented label $y'_i = \hat{g}(x'_i) + \hat{z}_i$ is a sample whose conditional noise distribution matches the original's by Assumption~\ref{assump:cond}. No causal claim is required to motivate the procedure.

\emph{Counterfactual view (analogy under additive noise).} When the decomposition $Y = g(X_P, X_R) + Z$ acts as an additive-noise SCM and $X_P$ is exogenous with respect to $Z$, reusing $\hat{z}_i$ while changing $x_{i,P}$ to $x'_{i,P}$ mirrors \citet{pearl2009causality}'s abduction--action--prediction template (Section~\ref{sec:background_causal}): abduct $\hat{z}_i$ from the observed sample, act on $X_P$, predict with the same noise. 
In Appendix~\ref{app:causal} we establish a precise relationship between CRDA and counterfactuals in Pearl's sense. Assuming that the true causal graph satisfies the d-separation analogue of Assumption~\ref{assump:cond} and that the noise term $Z$ is the only unobserved parent of $Y$, CRDA estimates the counterfactual probability $P(Y_{X=x'} \mid X=x, Y=y)$ at the population level. The finite-sample caveats of Section~\ref{sec:empirical_residuals} continue to apply.

\textbf{Reusing the Residual.} Setting $y'_i = \hat{g}(x'_i)$ alone would implicitly fix the noise to zero for every synthetic sample, teaching the model that the data manifold is deterministic and erasing the natural heteroscedasticity of the target. Reusing $\hat{z}_i$ preserves the instance-specific noise structure and distinguishes CRDA from naive feature perturbation.

\textbf{Residual Invariance Filters.}
The practical steps in Algorithm~\ref{alg:crda}, such as the PC algorithm and correlation checks, are \textit{empirical screens} designed to identify a feature subset $X_P$ for which Assumption \ref{assump:cond} is likely to hold. The choice of the filters themselves is a design decision and can be swapped for similar techniques. In Appendix~\ref{app:divergence}, we empirically validate our filters by showing that rejected features possess significantly higher Mutual Information with the residuals than selected features. When we reuse a specific residual $\hat{z}_i$ to construct a counterfactual label $y' = \hat{g}(x') + \hat{z}_i$, we make a practical choice that avoids the need to explicitly model the entire noise distribution.

\textbf{Model-Agnostic Nature.}
Although our algorithm is illustrated with a neural or tree-based baseline, the same counterfactual logic applies to \emph{any} parametric or nonparametric regressor: each training outcome is viewed as a sum of a learned systematic component and a noise term, and perturbations occur in the input space while the residual stays anchored to its original data point. While CRDA can be applied to any regression model, it does not \emph{guarantee} improvement on every such model. To guard against cases where augmentation would not help, CRDA defaults to the untouched baseline whenever its safeguards trigger.

\section{Experiments}
\label{sec:experiments}

\subsection{Experimental Setup and Protocol}

\textbf{High-Level Goal.}
Our primary goal is to investigate whether CRDA can reduce test MSE on tabular regression tasks, particularly under conditions of data scarcity. We evaluate performance across nine benchmark datasets and a synthetic task with a known ground-truth DGP.

\textbf{Comparisons and Baselines.}
\begin{enumerate}
    \item \textbf{No Augmentation (Baseline)}: Train the model on the raw data only.
    \item \textbf{CRDA Augmentation}: Train the same model class on the union of the raw data and the counterfactual samples generated by CRDA.
    \item \textbf{Generative Model Augmentation}: The baseline model architecture trained on the union of the original data and synthetic samples produced by state-of-the-art tabular data generators
\end{enumerate}

\textbf{Hyperparameter Search.}
We tune \emph{MLPRegressor} and \emph{XGBoostRegressor} with scikit-learn's \texttt{RandomizedSearchCV} \citep{pedregosa2011scikit} (3-fold) for 20 trials each, totaling \textit{60 fits} per baseline. \emph{CRDA} also introduces three additional knobs: \textsc{AugDataSizeFactor}, \textsc{PerturbationRange} and \textsc{MaxNumFeaturesToPerturb}. We tune these via \emph{Optuna} (TPE sampler) \citep{akiba2019optuna} for up to \textit{30} trials, similarly minimizing validation error \footnote{Complete search spaces and selected hyperparameters for the base regressors, together with the configurations used for all augmentation baselines, are detailed in Appendix~\ref{app:hparams}.}.

\textbf{Datasets and Preprocessing.}
We consider nine regression datasets from the University of California, Irvine (UCI) Machine Learning Repository, the Penn Machine Learning Benchmarks (PMLB) collection, and Kaggle \footnote{Summary statistics and sources of datasets are in Appendix~\ref{app:data}.}.
These were chosen to represent a variety of numeric, tabular domains and sample sizes. We drop duplicate rows and NaN values, then apply standardization per feature. For each dataset, we produce five training subsets, ranging from $n/5$ up to $n$. All subsets are split 80--20 into training and test sets.

\textbf{Evaluation Metrics and Significance Tests.}
We report the \textbf{MSE} (mean-squared error on the held-out test set) for our settings and their relative change (negative values indicate improvement). 
$$
\boldsymbol{\Delta\%}=100\,(\text{MSE}_{\mathrm{CRDA}}-\text{MSE}_{\mathrm{baseline}})/\text{MSE}_{\mathrm{baseline}}\
$$

Significance is assessed with a Wilcoxon signed-rank test across 10 CV folds \citep{wilcoxon1945individual}.

\subsection{Results and Analysis}

Our main findings are presented in Table~\ref{tab:baseline_results} and summarized in Figure~\ref{fig:avg_change_mse}. We see that CRDA provides consistent and substantial reductions in test MSE across nearly all datasets and training set sizes.
As shown in Figure~\ref{fig:avg_change_mse}, the \textbf{MLP Regressor benefits most significantly}, achieving an average MSE reduction of \textbf{22.9\%} across all nine datasets. This highlights CRDA's ability to provide valuable signal for data-hungry neural models. For instance, on the \textit{Parkinson's Monitoring} and \textit{House Price} datasets, MLP models augmented with CRDA see their error reduced by over 30\%. The \textbf{XGB Regressor} also shows consistent improvement, with an average MSE reduction of \textbf{6.4\%}.

Table~\ref{tab:baseline_results} details the performance at different data scales, averaged on 15 seeds \footnote{Full per-seed results with standard errors are in Appendix~\ref{app:perseed}}. Green cells indicate cases in which the Wilcoxon signed-rank test would have approved augmentation, and red cells indicate cases where it would have abstained. This built-in safeguard prevents CRDA from being applied where it might not be beneficial. While infrequent, we do note that this filter can be prone to error, particularly depending on sample size. Per-cell significance heat-maps are provided in Appendix~\ref{app:wilcoxon}.

\begin{figure*}[t]
  \centering
  \includegraphics[width=0.89\linewidth]{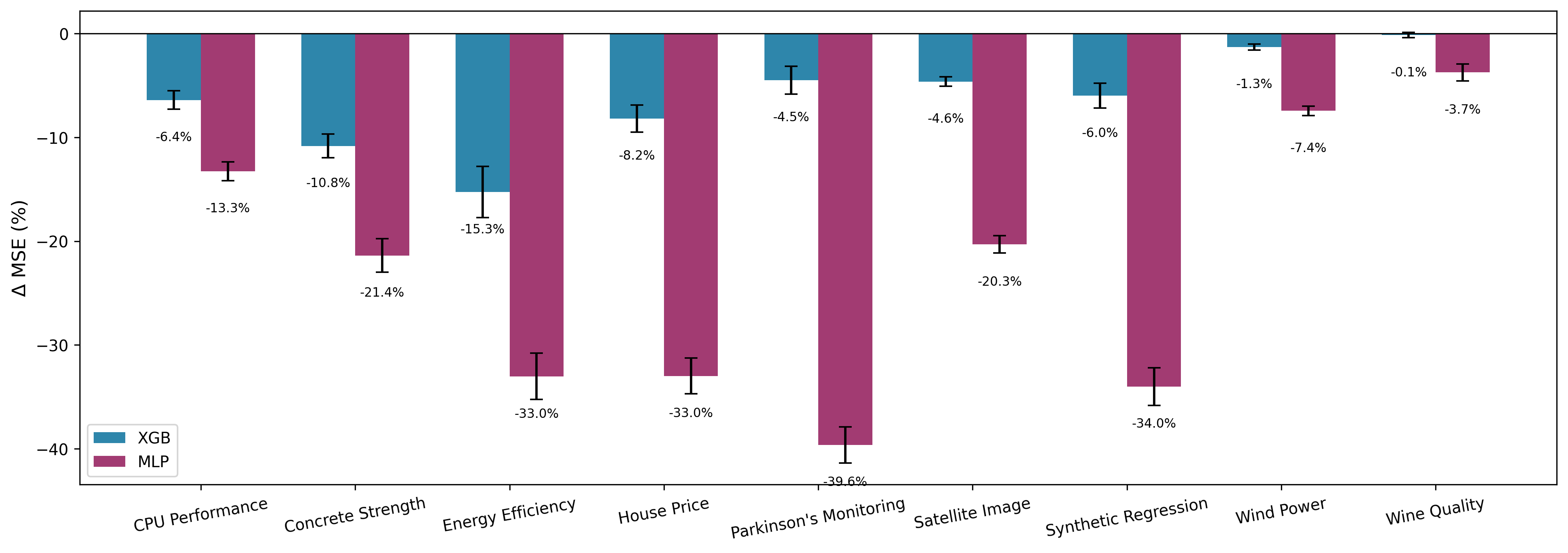}
  \caption{MSE percentage change for each dataset averaged over the five different training-subset sizes reported in Table~\ref{tab:baseline_results} with error bars corresponding to standard error. Lower is better.}
  \label{fig:avg_change_mse}
\end{figure*}

\begin{table}[!h]
  \scriptsize
\renewcommand{\arraystretch}{0.83}
  \centering
  \setlength\tabcolsep{0.95pt}
  \caption{Average change in MSE after applying CRDA for XGB and MLP across 15 seeds.
           Cells are green when data augmentation
           was selected to proceed more often based on the Wilcoxon signed rank test and red otherwise. The reported $\Delta$ values are computed with augmentation applied in all cases, independent of this decision. Lower is better $\downarrow$.}
  \label{tab:baseline_results}
  \begin{tabular}{l c c c}
    \toprule
    \textbf{Dataset} & \textbf{Sample Size} & \textbf{XGB MSE} $\Delta$\,\% & \textbf{MLP MSE} $\Delta$\,\% \\
    \midrule

% ---------------- CPU Performance ------------------------------------
\multirow{5}{*}{\shortstack[l]{CPU Performance \\ \citep{pmlb_227_cpu_small}}}  &  1638  &  \good{-6.99}  &  \good{-20.24} \\
    &  3276  &  \good{-9.47}  &  \good{-14.03} \\
    &  4914  &  \good{-6.20}  &  \good{-11.31} \\
    &  6552  &  \good{-4.13}  &  \good{-10.48} \\
    &  8190  &  \good{-5.19}  &  \good{-10.23} \\
% ---------------- Satellite Image ------------------------------------
\midrule
\multirow{5}{*}{\shortstack[l]{Satellite Image \\ \citep{pmlb_294_satellite_image}}}  &  1287  &  \good{-4.54}  &  \good{-18.36} \\
    &  2574  &  \good{-3.73}  &  \good{-16.69} \\
    &  3861  &  \good{-4.79}  &  \good{-23.14} \\
    &  5148  &  \good{-4.73}  &  \good{-23.72} \\
    &  6435  &  \good{-5.31}  &  \good{-19.66} \\
% ---------------- Wind Power ------------------------------------
\midrule
\multirow{5}{*}{\shortstack[l]{Wind Power \\ \citep{haslett1989_wind}}}  &  1314  &  \good{-2.82}  &  \good{-7.22} \\
    &  2628  &  \good{0.20}  &  \good{-9.17} \\
    &  3942  &  \good{-1.33}  &  \good{-9.03} \\
    &  5256  &  \good{-1.40}  &  \good{-6.15} \\
    &  6570  &  \good{-1.08}  &  \bad{-5.56} \\
% ---------------- Synthetic Regression ------------------------------------
\midrule
\multirow{5}{*}{\shortstack[l]{Synthetic Regression \\ \citep{pmlb_623_fri_c4_1000_10}}}  &  200  &  \good{-12.00}  &  \good{-28.80} \\
    &  400  &  \good{-3.16}  &  \good{-36.93} \\
    &  600  &  \good{-7.94}  &  \good{-27.91} \\
    &  800  &  \good{-2.23}  &  \good{-34.12} \\
    &  1000  &  \good{-4.59}  &  \good{-42.33} \\
% ---------------- Concrete Strength ------------------------------------
\midrule
\multirow{5}{*}{\shortstack[l]{Concrete Strength \\ \citep{uci_concrete_yeh1998}}}  &  201  &  \good{-8.01}  &  \good{-17.80} \\
    &  402  &  \good{-8.43}  &  \good{-19.83} \\
    &  603  &  \good{-9.75}  &  \good{-17.64} \\
    &  804  &  \good{-15.72}  &  \good{-24.77} \\
    &  1005  &  \good{-12.19}  &  \good{-26.90} \\
% ---------------- Energy Efficiency ------------------------------------
\midrule
\multirow{5}{*}{\shortstack[l]{Energy Efficiency \\ \citep{tsanas2012_energy}}}  &  153  &  \good{-13.33}  &  \good{-25.10} \\
    &  306  &  \good{-12.20}  &  \good{-28.13} \\
    &  459  &  \good{-10.55}  &  \good{-42.98} \\
    &  612  &  \good{-19.35}  &  \good{-40.71} \\
    &  765  &  \good{-20.96}  &  \good{-28.31} \\
% ---------------- HousePrice ------------------------------------
\midrule
\multirow{5}{*}{\shortstack[l]{House Price \\ \citep{kaggle2024_houseprice}}}  &  200  &  \good{-14.23}  &  \good{-40.57} \\
    &  400  &  \good{-5.39}  &  \good{-37.02} \\
    &  600  &  \good{-4.87}  &  \good{-30.14} \\
    &  800  &  \good{-9.86}  &  \good{-30.32} \\
    &  1000  &  \good{-6.50}  &  \good{-26.97} \\
% ---------------- Parkinson's Monitoring ------------------------------------
\midrule
\multirow{5}{*}{\shortstack[l]{Parkinson's Monitoring \\ \citep{tsanas2009_parkinsonstele}}}  &  1175  &  \good{-8.40}  &  \good{-36.17} \\
    &  2350  &  \good{-6.60}  &  \good{-31.82} \\
    &  3525  &  \good{-2.79}  &  \good{-36.60} \\
    &  4700  &  \good{-6.26}  &  \good{-46.40} \\
    &  5875  &  \bad{1.65}  &  \good{-47.23} \\
% ---------------- WineQuality ------------------------------------
\midrule
\multirow{5}{*}{\shortstack[l]{Wine Quality \\ \citep{cortez2009_winequality}}}  &  1063  &  \good{0.31}  &  \good{-0.34} \\
    &  2126  &  \good{1.01}  &  \good{-5.24} \\
    &  3189  &  \good{-0.33}  &  \good{-3.63} \\
    &  4252  &  \good{-0.61}  &  \good{-4.44} \\
    &  5315  &  \good{-1.08}  &  \good{-4.99} \\
    \bottomrule
  \end{tabular}
\end{table}

We further validate the necessity of our specific design choices in Appendix~\ref{app:simple_ablations}, showing that CRDA outperforms naive global perturbation and label-invariant baselines.

To better understand how CRDA's effectiveness scales with sample size in a controlled setting, we conducted an experiment on a synthetic DGP with a known ground-truth independence structure:
$$
Y = X_1^2 + X_2 X_3 + Z, \qquad \text{where } Z \perp (X_1, X_2, X_3)
$$
We generated 50,000 samples and applied CRDA at various sample sizes. The results, shown in Figure~\ref{fig:sample-size-synthetic}, reveal that at very low sample sizes ($<$2.5k), CRDA offers minimal benefit because the base predictor is too inaccurate to produce meaningful residuals. Conversely, at very high sample sizes ($>$30k), the baseline model is already so accurate that there is little room for improvement. The \textit{greatest MSE reduction occurs in a ``sweet spot"} (between 2.5k and 20k samples), where the baseline model has learned the main signal but still benefits from the localized exploration of the feature space that CRDA provides. 
This dependency on model capacity is further validated in Appendix~\ref{app:catboost} using \emph{CatBoost} \citep{prokhorenkova2018catboost}. Being more data-efficient, CatBoost's utility peaks at even lower sample sizes (peaking around $N \approx 500$ for datasets like \textit{Parkinson's} and \textit{Energy Efficiency}), confirming that the optimal CRDA window shifts earlier as the base learner becomes stronger.

\begin{figure*}[t]
  \centering
  \includegraphics[width=0.89\linewidth]{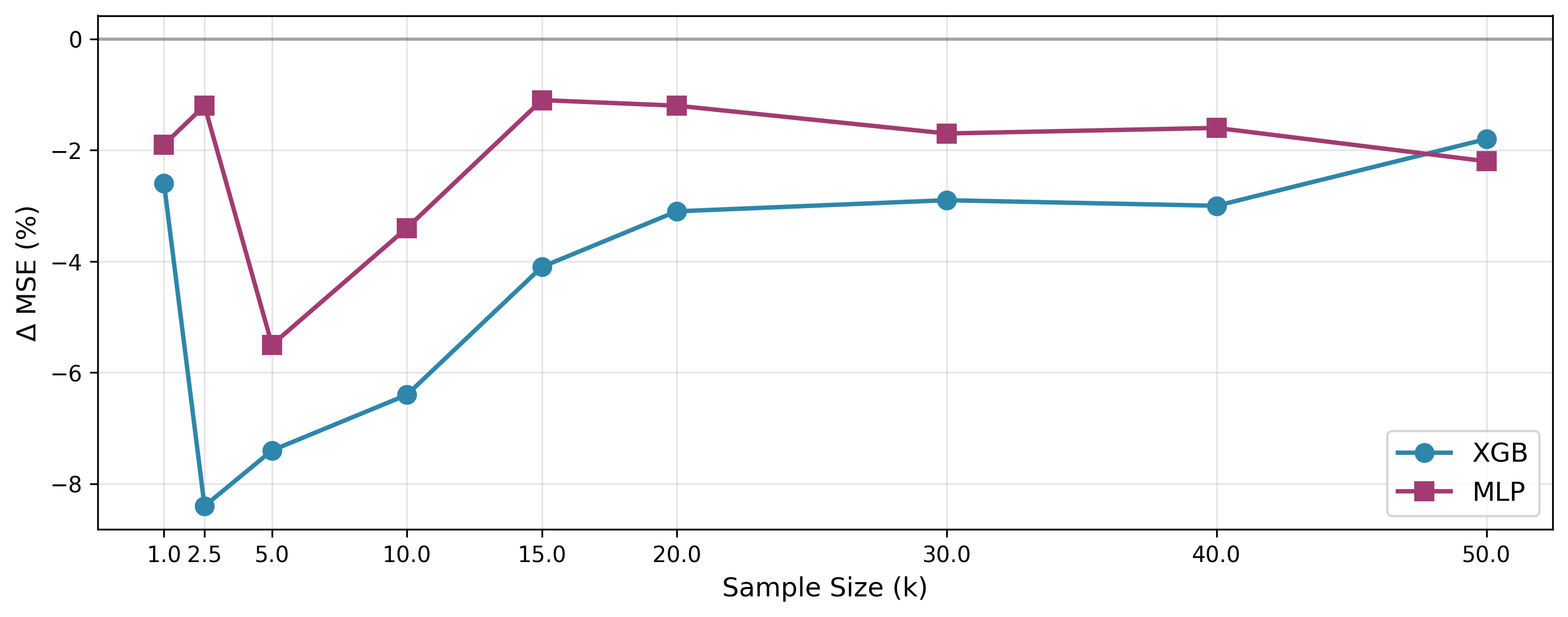}
  \caption{Synthetic sample-size-scaling experiment with a DGP of known independence for the residuals $Z$ and features $X$. For both XGB and MLP base models, we observe a ``sweet spot" where CRDA yields the largest MSE reduction (typically at a lower sample size).}
  \label{fig:sample-size-synthetic}
\end{figure*}

\begin{table*}[!h]
    \centering
    \caption{The percent MSE change for XGB and MLP base regressors. We compare CRDA against specialized regression augmentations (C-Mixup \citep{yao2022cmixup}, ADA \citep{schneider2023anchor}) and generative models (TabDDPM \citep{kotelnikov2023tabddpm}, TVAE \citep{xu2019ctgan}, CTGAN \citep{xu2019ctgan}). Averaged across 10 seeds, reporting standard error. Lower is better $\downarrow$.}
    \label{tab:data_generators}
    \resizebox{\textwidth}{!}{%
    \begin{tabular}{llcccccc}
      \toprule
      \multirow{2}{*}{\textbf{Dataset}} & \multirow{2}{*}{\textbf{Model}} & \multicolumn{6}{c}{\textbf{\% MSE Change} $\downarrow$} \\
      \cmidrule(lr){3-8}
       & & $\Delta_{\mathrm{C-Mixup}}$ & $\Delta_{\mathrm{ADA}}$ & $\Delta_{\mathrm{TabDDPM}}$ & $\Delta_{\mathrm{TVAE}}$ & $\Delta_{\mathrm{CTGAN}}$ & $\Delta_{\mathrm{CRDA}}$ \\
      \midrule
  
  % ---------------- CPU Performance ------------------------------------
  \multirow{2}{*}{\makecell[l]{CPU Performance \\ \citep{pmlb_227_cpu_small}}}
   & XGB & 1.7 $\pm$ 1.5 & 1.9 $\pm$ 0.6 & 36.5 $\pm$ 4.0 & 24.2 $\pm$ 2.6 & 41.6 $\pm$ 2.8 & \textbf{-1.4 $\pm$ 0.7} \\
   & MLP & -0.9 $\pm$ 1.0 & -0.6 $\pm$ 1.0 & 30.2 $\pm$ 7.1 & 36.6 $\pm$ 4.6 & 108.0 $\pm$ 13.3 & \textbf{-12.2 $\pm$ 1.0} \\
  \cmidrule(l){2-8}
  % ---------------- Satellite Image ------------------------------------
  \multirow{2}{*}{\makecell[l]{Satellite Image \\ \citep{pmlb_294_satellite_image}}}
   & XGB & 6.4 $\pm$ 1.7 & 1.4 $\pm$ 1.1 & 10.7 $\pm$ 1.4 & 12.5 $\pm$ 2.1 & 11.0 $\pm$ 1.1 & \textbf{-0.7 $\pm$ 0.7} \\
   & MLP & -0.8 $\pm$ 2.0 & 3.3 $\pm$ 2.4 & 9.9 $\pm$ 3.0 & 18.8 $\pm$ 3.3 & 55.1 $\pm$ 3.2 & \textbf{-23.5 $\pm$ 1.5} \\
  \cmidrule(l){2-8}
  % ---------------- Wind Power ------------------------------------
  \multirow{2}{*}{\makecell[l]{Wind Power \\ \citep{haslett1989_wind}}}
   & XGB & -1.9 $\pm$ 0.6 & -0.2 $\pm$ 0.3 & -0.4 $\pm$ 0.5 & 4.8 $\pm$ 0.9 & 7.1 $\pm$ 0.9 & \textbf{-2.6 $\pm$ 0.3} \\
   & MLP & 4.9 $\pm$ 3.4 & 14.7 $\pm$ 2.0 & -7.1 $\pm$ 1.4 & 5.5 $\pm$ 1.9 & 13.7 $\pm$ 1.9 & \textbf{-8.0 $\pm$ 1.1} \\
  \cmidrule(l){2-8}
  % ---------------- Synthetic Regression ------------------------------------
  \multirow{2}{*}{\makecell[l]{Synthetic Regression \\ \citep{pmlb_623_fri_c4_1000_10}}}
   & XGB & 141.5 $\pm$ 31.8 & 18.3 $\pm$ 3.6 & 25.0 $\pm$ 8.2 & 116.5 $\pm$ 20.2 & 164.1 $\pm$ 25.8 & \textbf{2.3 $\pm$ 1.8} \\
   & MLP & 78.1 $\pm$ 19.2 & 16.7 $\pm$ 6.4 & -19.2 $\pm$ 2.7 & 71.6 $\pm$ 15.4 & 198.9 $\pm$ 27.2 & \textbf{-33.3 $\pm$ 3.8} \\
  \cmidrule(l){2-8}
  % ---------------- Concrete Strength ------------------------------------
  \multirow{2}{*}{\makecell[l]{Concrete Strength \\ \citep{uci_concrete_yeh1998}}}
   & XGB & \textbf{-2.8 $\pm$ 1.7} & -0.1 $\pm$ 1.4 & -1.1 $\pm$ 3.0 & 7.9 $\pm$ 2.7 & 26.1 $\pm$ 4.3 & -1.7 $\pm$ 1.9 \\
   & MLP & -4.8 $\pm$ 2.9 & -3.8 $\pm$ 1.5 & -5.9 $\pm$ 2.9 & 36.8 $\pm$ 8.2 & 142.3 $\pm$ 17.6 & \textbf{-15.4 $\pm$ 2.3} \\
  \cmidrule(l){2-8}
  % ---------------- Energy Efficiency ------------------------------------
  \multirow{2}{*}{\makecell[l]{Energy Efficiency \\ \citep{tsanas2012_energy}}}
   & XGB & -18.0 $\pm$ 9.1 & -20.7 $\pm$ 3.5 & 3.4 $\pm$ 7.7 & -19.4 $\pm$ 8.1 & \textbf{-24.7 $\pm$ 6.2} & -10.7 $\pm$ 3.8 \\
   & MLP & 11.9 $\pm$ 14.2 & -22.9 $\pm$ 4.1 & 11.1 $\pm$ 11.1 & 127.7 $\pm$ 33.6 & 360.2 $\pm$ 51.5 & \textbf{-32.5 $\pm$ 5.4} \\
  \cmidrule(l){2-8}
  % ---------------- House Price ------------------------------------
  \multirow{2}{*}{\makecell[l]{House Price \\ \citep{kaggle2024_houseprice}}}
   & XGB & -12.8 $\pm$ 12.6 & \textbf{-42.9 $\pm$ 7.4} & -13.4 $\pm$ 12.9 & 298.6 $\pm$ 94.8 & 824.1 $\pm$ 139.6 & -12.8 $\pm$ 3.6 \\
   & MLP & -51.0 $\pm$ 5.1 & \textbf{-52.6 $\pm$ 5.8} & -4.3 $\pm$ 16.1 & 992.5 $\pm$ 263.7 & 3032.9 $\pm$ 374.3 & -42.3 $\pm$ 3.4 \\
  \cmidrule(l){2-8}
  % ---------------- Parkinson's Monitoring ------------------------------------
  \multirow{2}{*}{\makecell[l]{Parkinson's Monitoring \\ \citep{tsanas2009_parkinsonstele}}}
   & XGB & 105.1 $\pm$ 13.2 & 89.5 $\pm$ 11.1 & 286.7 $\pm$ 25.9 & 405.1 $\pm$ 53.2 & 565.3 $\pm$ 49.3 & \textbf{-0.3 $\pm$ 1.9} \\
   & MLP & 116.1 $\pm$ 37.5 & 28.7 $\pm$ 15.1 & 184.9 $\pm$ 36.3 & 659.0 $\pm$ 109.0 & 1454.2 $\pm$ 221.8 & \textbf{-48.2 $\pm$ 6.0} \\
  \cmidrule(l){2-8}
  % ---------------- Wine Quality ------------------------------------
  \multirow{2}{*}{\makecell[l]{Wine Quality \\ \citep{cortez2009_winequality}}}
   & XGB & -2.0 $\pm$ 0.5 & -0.0 $\pm$ 0.6 & -2.6 $\pm$ 0.7 & -0.5 $\pm$ 0.6 & 0.3 $\pm$ 0.8 & \textbf{-2.8 $\pm$ 0.4} \\
   & MLP & 13.1 $\pm$ 5.2 & 23.6 $\pm$ 2.5 & \textbf{-5.5 $\pm$ 0.6} & -0.8 $\pm$ 0.9 & 2.1 $\pm$ 0.6 & -2.9 $\pm$ 0.8 \\
  
      \bottomrule
    \end{tabular}
  }
  \end{table*}

Finally, we benchmarked CRDA against a comprehensive suite of baselines, including  regression augmentation methods (C-Mixup \citep{yao2022cmixup}, ADA \citep{schneider2023anchor}) and deep generative models (TabDDPM \citep{kotelnikov2023tabddpm}, TVAE \citep{xu2019ctgan}, CTGAN \citep{xu2019ctgan}). As shown in Table~\ref{tab:data_generators}, CRDA demonstrates superior stability and performance. For a fair comparison, CRDA's safety gate is disabled so that augmentation is applied unconditionally, exactly as the competing methods are; its advantage therefore comes from the augmentation mechanism itself.
While geometric methods like ADA and C-Mixup provide gains in specific settings, they exhibit catastrophic failure modes in others (e.g., increasing MSE by over 100\% on \textit{Synthetic Regression} and \textit{Parkinson's}). Similarly, deep generative models significantly degrade performance more often, likely due to difficulties in capturing the precise conditional distribution $P(Y|X)$ required for regression.
In contrast, CRDA's residual-preserving mechanism ensures that synthetic samples remain faithful to the underlying noise structure. Across all datasets, CRDA is the only method that reliably improves performance for both XGBoost and MLP models without the risk of significant degradation.

\section{Limitations}
\label{sec:limitations}

CRDA is currently designed for regression tasks; extending its principles to classification, where residuals are not straightforwardly defined, is a direction for future work.

The core of our method's validity hinges on a key assumption: the model's residual noise, $Z$, is conditionally independent of the features we choose to perturb, $X_P$, given the features we hold fixed, $X_R$ (Assumption \ref{assump:cond}). In practice, verifying this assumption from finite data is a primary challenge. A poorly fitted base predictor can produce residuals that retain dependencies on all input features, causing interventions on $X$ to break the required noise invariance. To address this, CRDA employs a two-layer screen. 

First, we use the PC algorithm and a Pearson correlation test as a \textit{risk-control heuristic} to filter for candidate features that are likely to satisfy Assumption~\ref{assump:cond}. We acknowledge this screen is imperfect; the PC algorithm can fail in the presence of unobserved confounders and scale poorly with high feature counts, while correlation tests may not detect non-linear dependencies.
However, these filters are not intended to be infallible but rather a practical first line of defense. The theoretical guarantees of the PC algorithm, for instance, are well-studied; under standard assumptions, its error probability of incorrectly identifying an edge decays exponentially with sample size \citep{kalisch2007estimating}. This sample consistency suggests that the risk of our filter admitting a feature that violates Assumption~\ref{assump:cond} diminishes as more data becomes available. 

Moreover, CRDA incorporates a second, decisive \textit{safety gate}: the Wilcoxon signed-rank test. This evaluates the realized impact of augmentation on a validation set. If the generated samples do not yield a statistically significant improvement, we return the original baseline. This mechanism ensures that we either improve the baseline or abstain from augmentation, thereby mitigating the risk of performance degradation from an imperfect initial screen.

To address concerns that a strong base learner is required, we include \emph{linear regression} experiments in Appendix \ref{app:linreg}. In every dataset/fold, the Wilcoxon gate produced $p\!>\!0.05$, so CRDA \emph{abstained}. If one \emph{ignores} the gate and forces augmentation, performance generally degrades or does not improve, illustrating that the safety checks are beneficial and block harmful augmentation for weak baselines.

Finally, the performance of CRDA is sensitive to both dataset size and the choice of the base predictor. For very large datasets, the need for augmentation diminishes, and CRDA offers little benefit. Conversely, if a dataset is too small, the base model may be too weak to produce meaningful residuals that are even approximately independent of the features, and our statistical tests will lack power. This can be seen in our sample-size-scaling experiment in Figure \ref{fig:sample-size-synthetic}.

\section{Conclusion}
\label{sec:conclusion}

We described a new data augmentation technique for regression called CRDA. CRDA is model-agnostic and it does not assume any domain knowledge such as specific transformations that preserve labels. Instead, it leverages counterfactual reasoning and the invariance of the residual noise distribution. We demonstrated the effectiveness of CRDA in data scarce regression tasks where it helped improve predictions made by representative base predictors including XGBoost and multi-layer perceptrons. We also displayed substantially stronger and more reliable results when compared to state-of-the-art tabular data generators.

Several directions for future work remain. The first is to extend CRDA to classification tasks. The key challenge is due to the non-numeric nature of residuals, though embedding-based transformations offer a potential path. The second would be to explore alternative methods in our feature partitioning step, such as formal proxy-based causal inference techniques or confounder-robust causal discovery algorithms like Fast Causal Inference (FCI)~\citep{spirtes2000causation}. This may enable CRDA to better adjust for hidden factors rather than simply discarding confounded features.

To support adoption on real tabular datasets, which often contain categorical columns outside the scope of the multiplicative perturbation analyzed here, we additionally release a reference implementation as a Python package. It is available via \texttt{pip install crda} (\url{https://pypi.org/project/crda/}) and extends the paper's method by supporting uniform category resampling for one-hot-encoded categorical features. We also replace our feature independence tests with distance-correlation \citep{szekely2007measuring} for continuous features and mutual information for categorical features. This variant still shares all other logic including the same residual-reuse construction and Wilcoxon outcome gate.

\section*{Impact Statement}
This is foundational work that aims to improve regression in data scarce scenarios, which can span a wide range of important domains.  As discussed in Section~\ref{sec:limitations} of the paper, CRDA could worsen predictive accuracy instead of improving it, leading to negative consequences in high impact applications.  To mitigate such outcomes, CRDA has filters in place that detect and prevent harm.

\section*{Acknowledgment}
We acknowledge funding from the Canada CIFAR AI Chair program, discovery grants to Poupart, Li and Schulte
    from the Natural Sciences and Engineering Research Council of Canada and a grant
    from IITP \& MSIT of Korea (No.~RS-2024-00457882, AI Research Hub Project).
    Computational resources used in preparing this research were provided, in part,
    by the Province of Ontario, the Government of Canada through CIFAR, and
    companies sponsoring the Vector Institute
\url{https://vectorinstitute.ai/about/current-partners/}.

\newpage
\clearpage

\bibliography{example_paper}
\bibliographystyle{icml2026}

%%%%%%%%%%%%%%%%%%%%%%%%%%%%%%%%%%%%%%%%%%%%%%%%%%%%%%%%%%%%%%%%%%%%%%%%%%%%%%%
%%%%%%%%%%%%%%%%%%%%%%%%%%%%%%%%%%%%%%%%%%%%%%%%%%%%%%%%%%%%%%%%%%%%%%%%%%%%%%%
% APPENDIX
%%%%%%%%%%%%%%%%%%%%%%%%%%%%%%%%%%%%%%%%%%%%%%%%%%%%%%%%%%%%%%%%%%%%%%%%%%%%%%%
%%%%%%%%%%%%%%%%%%%%%%%%%%%%%%%%%%%%%%%%%%%%%%%%%%%%%%%%%%%%%%%%%%%%%%%%%%%%%%%
\newpage
\appendix
\onecolumn

\section{Methodological Diagram}\label{app:pipeline_diagram}

% --- DIAGRAM CODE ---
\begin{figure}[htbp]
    \centering
    \begin{tikzpicture}[
        node distance=1.0cm and 1.0cm, % Vertical and Horizontal spacing
        auto,
        % --- STYLES ---
        block/.style={
            rectangle, 
            draw=black!60, 
            thick, 
            fill=white, 
            text width=3.5cm, 
            align=center, 
            rounded corners, 
            minimum height=1.2cm,
            drop shadow
        },
        data/.style={
            trapezium, 
            trapezium left angle=70, 
            trapezium right angle=110, 
            draw=black!60, 
            thick, 
            fill=yellow!10, 
            align=center, 
            drop shadow,
            minimum width=2.5cm
        },
        decision/.style={
            diamond, 
            draw=orange!60, 
            thick, 
            fill=orange!10, 
            text width=2.5cm, 
            align=center, 
            aspect=2, 
            drop shadow
        },
        line/.style={
            draw, 
            -Latex, 
            thick, 
            gray!80,
            rounded corners=5pt
        },
        % Highlights
        systematic/.style={draw=blue!70, fill=blue!5},
        residual/.style={draw=red!70, fill=red!5}
    ]

    % --- NODES ---

    % 1. Top: Input
    \node [data] (input) {Training Data \\ $\{(x_i, y_i)\}$};

    % 2. Baseline (Below Input)
    \node [block, below=0.6cm of input, systematic] (baseline) {\textbf{Base Predictor} \\ Train $\hat{g}(\cdot)$};

    % 3. Branching: Residuals (Left) and Filter (Right)
    \node [block, below left=1.2cm and 0.2cm of baseline, residual] (residuals) {\textbf{Residual \\ Extraction} \\ $\hat{z}_i = y_i - \hat{g}(x_i)$};
    
    \node [decision, below right=2.3cm and 0.2cm of baseline] (filter) {\textbf{Independence \\ Filter} \\ (PC Alg + Corr)};

    % 4. Action: Perturb (Below Filter)
    \node [block, below=0.8cm of filter] (perturb) {\textbf{Perturb Features} \\ $x'_P = x_P(1+\delta)$ \\ Keep $x_R$ Fixed};

    % 5. Synthesis: Repredict (Centered below Residuals and Perturb)
    % We use coordinate calculation to center it perfectly
    \node [block, below=2.5cm of baseline] (repredict) at ($(residuals.south)!0.5!(perturb.south) + (0,-1.5)$) {\textbf{Counterfactual \\ Labeling} \\ $y'_i = \hat{g}(x') + \hat{z}_i$};

    % 6. Bottom: Output
    \node [data, below=0.6cm of repredict, fill=green!10] (output) {Augmented Data \\ $\mathcal{D}_{\text{aug}}$};

    % --- CONNECTIONS ---

    % Input -> Baseline
    \path [line] (input) -- (baseline);

    % Baseline -> Branches
    % We create a split point below baseline
    \coordinate [below=0.6cm of baseline] (split);
    \path [line] (baseline) -- (split);
    
    % Split -> Residuals
    \path [line] (split) -| node[pos=0.3, above] {Predictions} (residuals);
    
    % Split -> Filter
    \path [line] (split) -| node[pos=0.3, above] {Features $X$} (filter);

    % Filter -> Perturb
    \path [line] (filter) -- node [right] {Select $X_P$} (perturb);

    % Perturb -> Repredict
    \path [line] (perturb.south) |- ++(0,-0.5) -| node [pos=0.1, below] {New Input $x'$} (repredict.north east);

    % Residuals -> Repredict
    \path [line] (residuals.south) |- ++(0,-0.5) -| node [pos=0.1, below] {Invariant $\hat{z}_i$} (repredict.north west);

    % Repredict -> Output
    \path [line] (repredict) -- (output);

    % Optional Background for Generation Phase
    \begin{scope}[on background layer]
        \node [draw=gray!20, dashed, rounded corners, fit=(filter) (perturb), label=right:\textcolor{gray}{\textit{Generation}}] {};
    \end{scope}

    \end{tikzpicture}
    \caption{The Counterfactual Residual Data Augmentation (CRDA) Pipeline. The workflow proceeds top-to-bottom; the base model component $\hat{g}(\cdot)$ first isolates the residual noise $\hat{z}_i$. Simultaneously, the Independence Filter identifies safe features $x_P$. These are perturbed and recombined with the preserved residual to generate valid counterfactual samples.}
    \label{fig:crda_pipeline}
\end{figure}
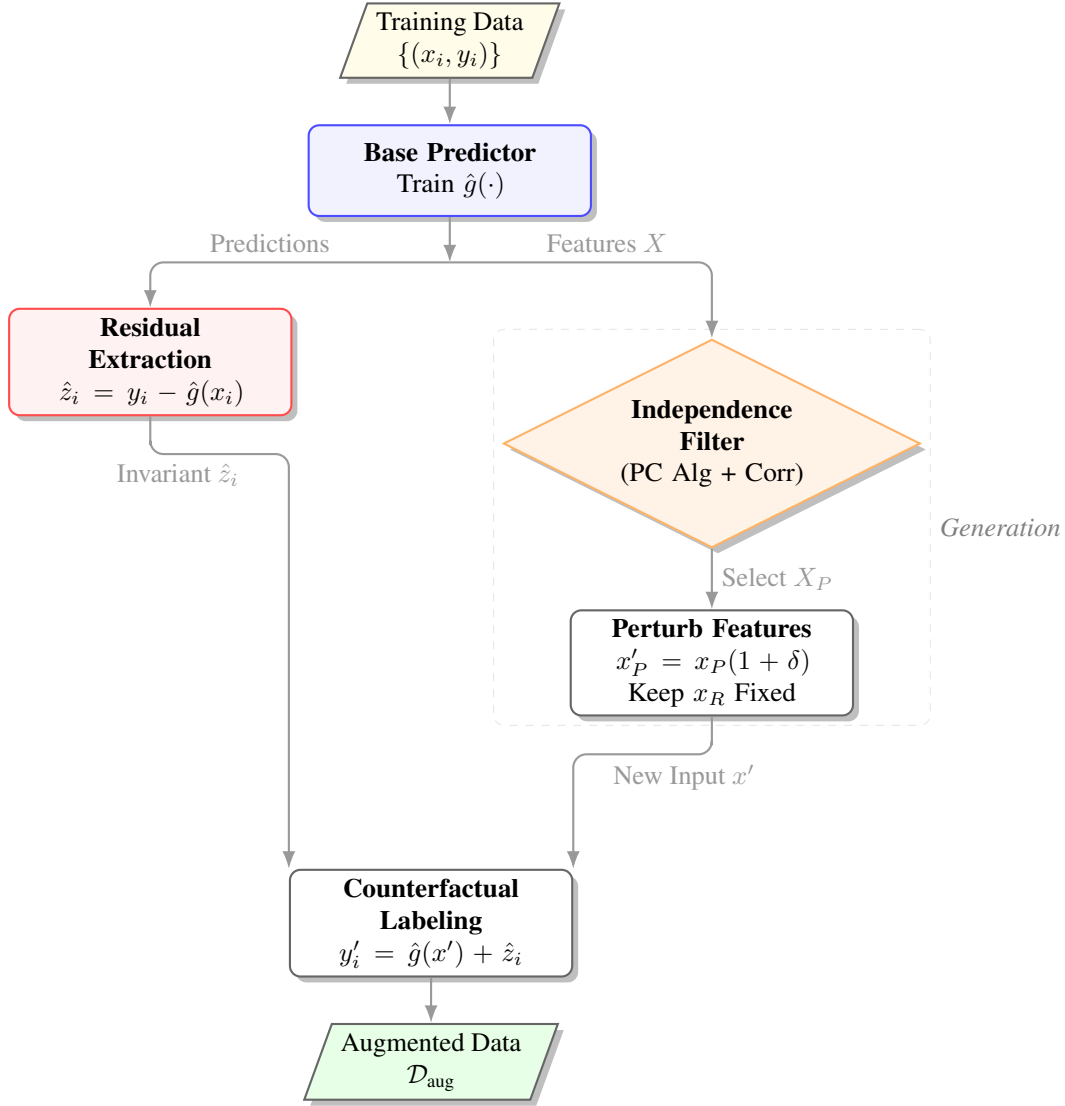

\section{Implementation Details}
\label{app:implementation}

All experiments were conducted in \texttt{Python}, leveraging standard libraries for machine learning and hyperparameter optimization. We used \texttt{scikit-learn} for the \texttt{MLPRegressor} baseline \citep{pedregosa2011scikit}, \texttt{XGBoost} for the \texttt{XGBoostRegressor} baseline \citep{chen2016xgboost}, and \texttt{Optuna} for tuning CRDA's specific hyperparameters \citep{akiba2019optuna}.

Our experimental protocol uses 10-fold cross-validation (CV), and we seed all random number generators to ensure reproducibility. Computations were performed on a single \textbf{AWS c7i.24xlarge} instance equipped with 96 vCPUs and 192 GB of RAM. To facilitate the complete reproduction of our findings, we have made the full study logs, JSON configuration files, and source code available in our project repository, \url{https://github.com/mhmohebbi/CRDA/}.

%-------------------------------------------------------------
\section{Dataset Summary}\label{app:data}
%-------------------------------------------------------------
Table~\ref{tab:dataset_stats} lists the nine tabular-regression benchmarks used in the paper, with their total sample count, dimensionality (\# numeric columns), and provenance
repository.
\begin{table}[h]
  \centering
  \caption{Basic statistics of the evaluation datasets, including number of samples ($n_{\text{samp}}$) and number of features ($n_{\text{features}}$).}
  \label{tab:dataset_stats}
  \begin{tabular}{lrrl}
    \toprule
    Dataset & $n_{\text{samp}}$ & $n_{\text{features}}$ & Source \\
    \midrule
    CPU Performance	               &  8192 & 12 & PMLB \citep{pmlb_227_cpu_small} \\
    Satellite Image	               &  6435 & 36 & PMLB \citep{pmlb_294_satellite_image} \\
    Wind Power                    &  6574 & 14 & UCI \citep{haslett1989_wind}  \\
    Synthetic Regression           &  1000 & 10 & PMLB \citep{pmlb_623_fri_c4_1000_10} \\
    Concrete Strength  &  1005 & 8 & UCI \citep{uci_concrete_yeh1998}  \\
    Energy Efficiency              &  768 & 9 & UCI \citep{tsanas2012_energy}  \\
    House Price                    &  1000 & 7 & Kaggle \citep{kaggle2024_houseprice} \\
    Parkinson's Monitoring	       &  5875 & 20 & UCI \citep{tsanas2009_parkinsonstele}  \\
    Wine Quality                   &  5318 & 11 & UCI \citep{cortez2009_winequality} \\
    \bottomrule
  \end{tabular}
\end{table}

%-------------------------------------------------------------
\section{Complete Experimental Protocol}\label{app:protocol}
%-------------------------------------------------------------
\paragraph{Configuration Objects.}
We centralize all experiment settings (e.g.\ dataset path, model type, global seed, CRDA knobs) in a Python class \texttt{Config}. Each run instantiates a \texttt{Config} with specific arguments and passes it to our \texttt{Experiment} harness, which saves the resulting configuration to a JSON file for reproducibility.

\begin{lstlisting}[language=json,caption={Example truncated config file for an XGB run.},label={lst:config}]
{
  "baseline": "xgboost",
  "dataset_path": "../data/WineQuality.csv",
  "sample_sizes": [1063, 2126, 3189, 4252, 5315],
  "ignore_filter": true,
  "hyperparam_tune": true,
  "results_dir": "../experiments/WineQuality",
  "...": "More fields omitted (test_size, num_seeds, p_wilcoxon_threshold, etc.)"
}
\end{lstlisting}

\paragraph{Key Fields and Usage}
\begin{itemize}
    \item \textbf{Model parameters:} \texttt{baseline} can be set to ``mlp'' or ``xgboost''; we do not alter other hyper-parameters (those are tuned via \texttt{RandomizedSearchCV}).
    \item \textbf{CRDA knobs:} \texttt{aug\_data\_size\_factor}, 
    \texttt{max\_n\_features\_to\_perturb} and
    \texttt{max\_perturb\_percent}. These are also tunable parameters (via \texttt{Optuna}). They specify how many counterfactual samples to generate, how many perturbable features we perturb and by how much; see Section~\ref{app:crda_knobs}.
    \item \textbf{Data splits:} \texttt{sample\_sizes} enumerates partial subsets of a dataset (e.g.\ $\frac{n}{5},\dots,n$), and \texttt{test\_size} sets the final train–test ratio.
    \item \textbf{Miscellaneous toggles:} \texttt{hyperparam\_tune} (whether to run a cross-validated search), \texttt{ignore\_filter} (bypass CRDA’s feature independence checks), \texttt{save\_plots}, etc.
\end{itemize}
For each experiment, the \texttt{Experiment} class reads the \texttt{config} object, runs the pipeline (training, augmentation, evaluation), and dumps logs plus final results in a timestamped directory. By reloading \texttt{config.json} via \texttt{Config.from\_dict}, one can exactly reproduce the same run.

%-------------------------------------------------------------
\section{Hyper-parameters and Search Spaces}\label{app:hparams}
%-------------------------------------------------------------
\subsection{Base Regressor Configurations}\label{app:predictor_hparams}

The two baseline families--\textbf{MLPRegressor} and \textbf{XGBoostRegressor}--share a
hybrid strategy: we \emph{fix} well-established architectural or optimisation knobs to
textbook defaults, while \emph{searching} over the handful of hyper-parameters that most
strongly drive bias–variance trade-offs.  This mirrors common practice in tabular ML
benchmarks \citep{delgado2014dozen} and keeps the search budget ($20$ trials per 3-fold, per dataset, per baseline) focused on the levers that matter.

\paragraph{Why these choices?}
For MLPs we retain the ReLU–Adam recipe that has been shown to be robust for small/medium
tabular tasks \citep{goodfellow2016dl}. We enable \emph{adaptive} learning-rate and early
stopping to guard against over-training, and explore only depth/width
(`hidden\_layer\_sizes`) and three learning-dynamics scalars
($\alpha$,~$\text{learning\_rate\_init}$,~$\text{tol}$).
For XGBoost we follow the histogram grow policy (``\texttt{tree\_method=hist}'') that is
memory-friendly on CPUs, fix $1\,000$ boosting rounds (with early stopping inside the CRDA
loop), and search the usual five knobs that govern tree shape, sampling and shrinkage.
Median wall-clock per trial on a \texttt{c7i.24xlarge} is
$\sim$9s (MLP) and $\sim$4s (XGB).\footnote{Full timing logs are available in
\texttt{experiments/full\_reproduction.ipynb} of the code repository.}
Tables \ref{tab:hparams_mlp} and \ref{tab:hparams_xgb} enumerate the
\textit{search priors} together with the \emph{modal} best value across
datasets.

\begin{table}[h]
  \centering
  \caption{MLPRegressor hyper-parameters searched with
           \textsc{RandomizedSearchCV}.  Ranges use log-uniform
           ($\mathrm{LogU}$) or categorical priors.}
  \label{tab:hparams_mlp}
  \begin{tabular}{lll}
    \toprule
      \textbf{Parameter} & \textbf{Prior / Range} & \textbf{Modal best} \\
    \midrule
      hidden\_layer\_sizes & $\{(128,64,32),(128,64),(64,32),(64,)\}$ & (128,64,32) \\
      $\alpha$ (L2)        & $\mathrm{LogU}(10^{-5},10^{-3})$         & 0.00040 \\
      learning\_rate\_init & $\mathrm{LogU}(10^{-3},10^{-2})$         & 0.00942 \\
      tol                  & $\mathrm{LogU}(10^{-5},10^{-4})$         & 0.00009 \\
    \midrule
      \multicolumn{3}{l}{\emph{Fixed for all runs}}\\
      activation           & \multicolumn{2}{l}{relu} \\
      solver               & \multicolumn{2}{l}{adam} \\
      batch\_size          & \multicolumn{2}{l}{32} \\
      max\_iter            & \multicolumn{2}{l}{1\,000} \\
      learning\_rate       & \multicolumn{2}{l}{adaptive} \\
      early\_stopping      & \multicolumn{2}{l}{true} \\
      validation\_fraction & \multicolumn{2}{l}{0.10} \\
      n\_iter\_no\_change   & \multicolumn{2}{l}{20} \\
    \bottomrule
  \end{tabular}
\end{table}

\begin{table}[h]
  \centering
  \caption{XGBoostRegressor hyper-parameters searched with
           \textsc{RandomizedSearchCV}.  Log-spaces are
           base-10.}
  \label{tab:hparams_xgb}
  \begin{tabular}{lll}
    \toprule
      \textbf{Parameter} & \textbf{Prior / Range} & \textbf{Modal best} \\
    \midrule
      learning\_rate      & $\log_{10}\!\,[10^{-3},10^{-1}]$ (10 pts) & 0.02154 \\
      max\_depth          & $\{3,4,6\}$                               & 6 \\
      min\_child\_weight  & $\{1,5\}$                                 & 5 \\
      subsample           & $\{0.7,1.0\}$                             & 0.7 \\
      colsample\_bytree   & $\{0.7,1.0\}$                             & 0.7 \\
      reg\_lambda         & $\log_{10}\!\,[10^{-3},10^{1}]$ (6 pts)   & 0.03981 \\
    \midrule
      \multicolumn{3}{l}{\emph{Fixed for all runs}}\\
      objective           & \multicolumn{2}{l}{reg:squarederror} \\
      tree\_method        & \multicolumn{2}{l}{hist} \\
      n\_estimators       & \multicolumn{2}{l}{1\,000} \\
      reg\_alpha          & \multicolumn{2}{l}{0.0} \\
      early\_stopping\_rounds & \multicolumn{2}{l}{20} \\
    \bottomrule
  \end{tabular}
\end{table}

\subsection{Data Augmentation Baseline Configurations}\label{app:baseline_hparams}

For comparability and reproducibility of the Table~\ref{tab:data_generators} benchmark, we report the full configuration used for each augmentation baseline. In all cases we followed the recommended defaults from the method's original paper and its public reference implementation; no method-specific tuning was performed on our datasets. Table~\ref{tab:baseline_hparams} summarizes the values.

\begin{table}[h]
  \centering
  \caption{Hyperparameter configurations for the augmentation baselines used in Table~\ref{tab:data_generators}. Values follow the recommended defaults of each method's reference implementation.}
  \label{tab:baseline_hparams}
  \begin{tabular}{lll}
    \toprule
      \textbf{Method} & \textbf{Hyperparameter} & \textbf{Value} \\
    \midrule
      \multirow{4}{*}{C-Mixup}
        & alpha          & 1.0 \\
        & KDE bandwidth  & 1.0 \\
        & mixtype        & \texttt{kde} \\
        & KDE type       & \texttt{gaussian} \\
    \midrule
      \multirow{2}{*}{ADA}
        & gammas             & $\{0.5,\,0.75,\,1.5,\,2.0\}$ \\
        & number of anchors  & 10 \\
    \midrule
      \multirow{7}{*}{TabDDPM}
        & learning rate         & $10^{-3}$ \\
        & batch size            & 4096 \\
        & diffusion timesteps   & 1000 \\
        & training iterations   & 1000 \\
        & MLP layers            & 2 \\
        & MLP layer width       & 256 \\
        & dropout               & 0.0 \\
    \midrule
      \multirow{6}{*}{CTGAN}
        & batch size                          & 500 \\
        & embedding dimension                 & 128 \\
        & generator / discriminator dims      & $(256,\,256)$ \\
        & learning rate (G / D)               & $2\!\times\!10^{-4}$ \\
        & weight decay                        & $10^{-6}$ \\
        & pac                                 & 10 \\
    \midrule
      \multirow{5}{*}{TVAE}
        & embedding dimension  & 128 \\
        & encoder dims         & $(128,\,128)$ \\
        & decoder dims         & $(128,\,128)$ \\
        & L2 scale             & $10^{-5}$ \\
        & loss factor          & 2 \\
    \bottomrule
  \end{tabular}
\end{table}

\newpage

%-------------------------------------------------------------
\section{CRDA Knob Selection \& Sensitivity}\label{app:crda_knobs}
%-------------------------------------------------------------
\vspace{-0.4em}
CRDA exposes three \emph{augmentation knobs}.  
During a 30-trial \textsc{Optuna}–TPE search
(per dataset, per baseline) we sample from the priors in
Table~\ref{tab:crda_knobs}; all other implementation details
are inherited from Algorithm~1 (main paper).

\begin{itemize}
\item \textbf{\texttt{max\_n\_features\_to\_perturb}} \;
      controls \emph{how many} invariant features are jointly
      edited, trading off sample realism against diversity.
\item \textbf{\texttt{aug\_data\_size\_factor}} \;
      decides the \# of counterfactuals per real point; values
      $<\!1$ mitigate class-imbalance–style bias, whereas
      $>\!1$ favors variance reduction.
\item \textbf{\texttt{max\_perturb\_percent}} \;
      sets the half-width of the \([-p,+p]\) uniform scaling
      band; larger $p$ injects broader \emph{counterfactual
      sweep} but risks violating local linearity assumptions
      of the residual.%
\end{itemize}

\smallskip
These three parameters explain the vast majority of
between-trial variance in validation MSE, so limiting
\textsc{Optuna} to a small budget remains effective.
Median trial time is $\sim$7.5s (MLP) and $\sim$3.7s (XGB).

Table~\ref{tab:crda_knobs} reports the \emph{modal} best-value
across the nine benchmarks and their partitions.

\begin{table}[h]
  \centering
  \caption{CRDA augmentation knobs: search priors and
           modal best values.}
  \label{tab:crda_knobs}
  \begin{tabular}{lll}
    \toprule
      \textbf{Knob} & \textbf{Search prior / range} & \textbf{Modal best} \\
    \midrule
      max\_n\_features\_to\_perturb & $\{1,2,3,4,5\}$         
      & 2 \\
      aug\_data\_size\_factor       & $\{0.50,0.75,1.00,1.25,1.50\}$ & 1.25 \\
      max\_perturb\_percent         & $\{0.10,0.20,\dots,1.00\}$      & 0.7 \\
    \bottomrule
  \end{tabular}
\end{table}

% .............................................................
\subsection*{One-dataset sweep (House Price)}
% .............................................................
For illustration, we fix two of the three knobs at their modal best values (from Table~\ref{tab:crda_knobs}) and systematically vary the remaining knob. Figures~\ref{fig:knob-sens-mlp} and \ref{fig:knob-sens-xgb} show the resulting percentage change in MSE (\textit{lower} is better) on the \textit{House Price} dataset, averaged over five random seeds. We make several observations:

\begin{itemize}
    \item \textbf{Augmenting data size} (\texttt{aug\_data\_size\_factor}) appears more beneficial for MLP, presumably because additional training samples reduce overfitting; by contrast, XGB sees weaker or even mixed effects here, consistent with the notion that tree ensembles can already leverage smaller sets effectively.
    \item \textbf{Number of perturbed features} (\texttt{max\_n\_features\_to\_perturb}) shows an opposite preference: XGB yields stronger gains when more features are jointly modified, whereas MLP performance degrades if we perturb too many simultaneously (likely hurting the local consistency of the residual).
    \item \textbf{Perturbation magnitude} (\texttt{max\_perturb\_percent}) also diverges across baselines: larger scales help XGB discover more diverse synthetic points, but MLP tends to prefer smaller shifts in order to maintain stable gradients in training.
\end{itemize}

In short, although \emph{both} models benefit from CRDA overall, their ideal hyper-parameter configurations differ. This shows the importance of model-aware tuning for effective data augmentation.

\newpage

\begin{figure}[!h]
  \centering
  
  \begin{subfigure}[t]{0.62\linewidth}
    \centering
    \includegraphics[width=\linewidth]{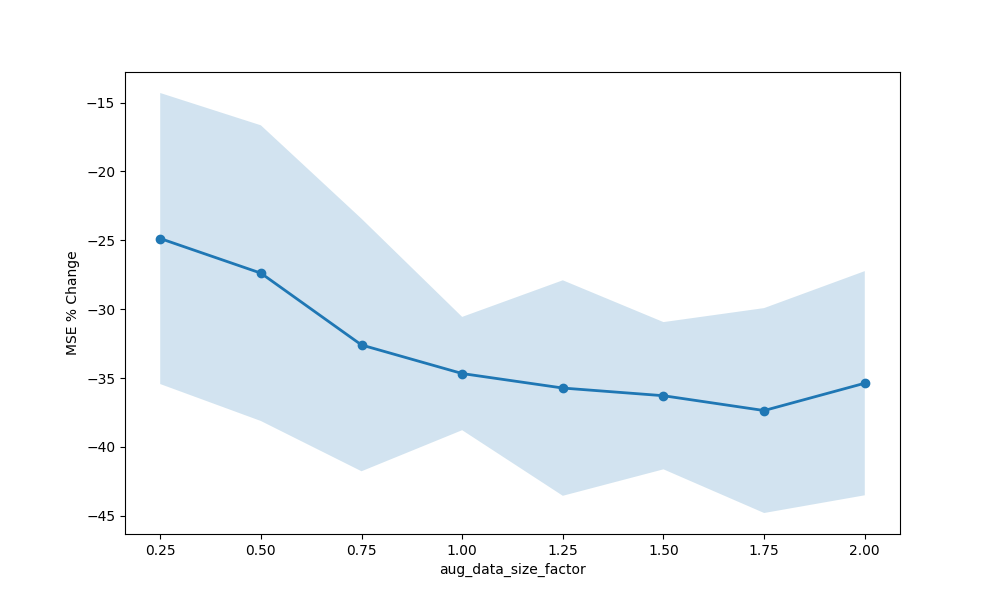}
    % \caption{\texttt{aug\_data\_size\_factor}}
  \end{subfigure}\hfill
  
  \begin{subfigure}[t]{0.62\linewidth}
    \centering
    \includegraphics[width=\linewidth]{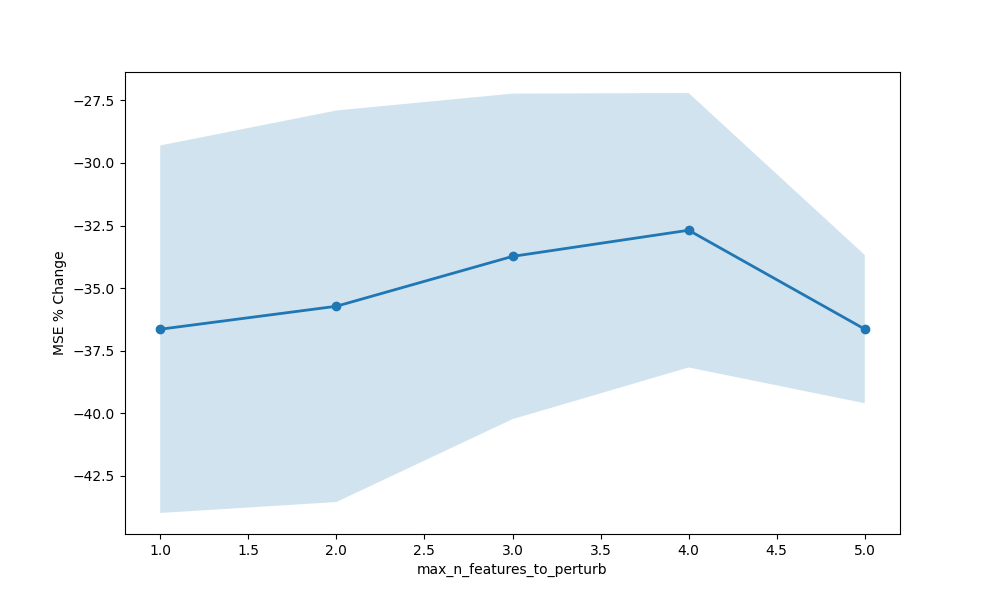}
    % \caption{\texttt{max\_n\_features\_to\_perturb}}
  \end{subfigure}\hfill

  \begin{subfigure}[t]{0.62\linewidth}
    \centering
    \includegraphics[width=\linewidth]{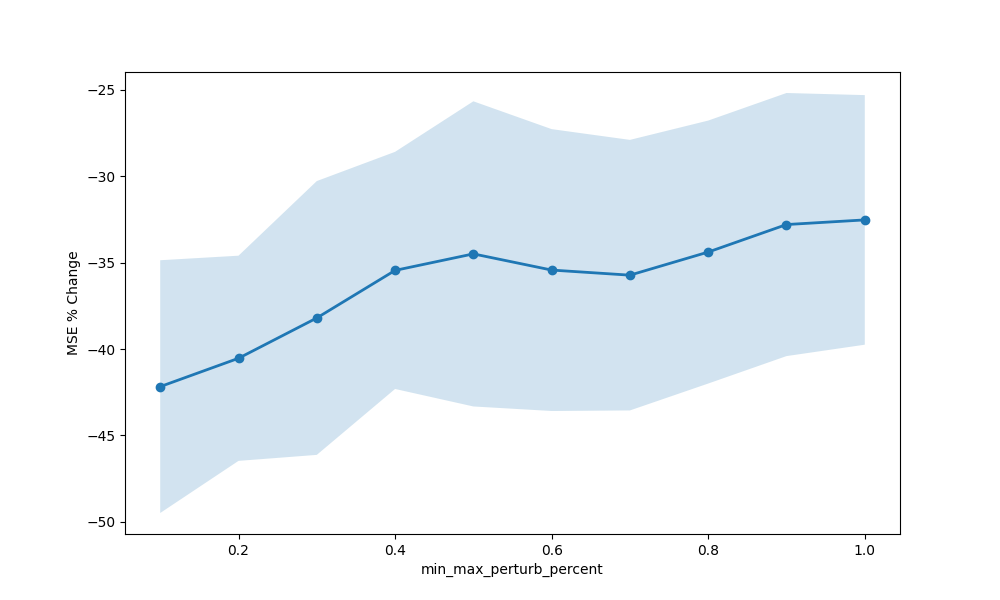}
    % \caption{\texttt{max\_perturb\_percent}}
  \end{subfigure}

  \caption{CRDA knob–sensitivity on the \textbf{MLP} baseline
           (House\-Price dataset).}
  \label{fig:knob-sens-mlp}
\end{figure}

\newpage

\begin{figure}[!h]
  \centering
  
  \begin{subfigure}[t]{0.62\linewidth}
    \centering
    \includegraphics[width=\linewidth]{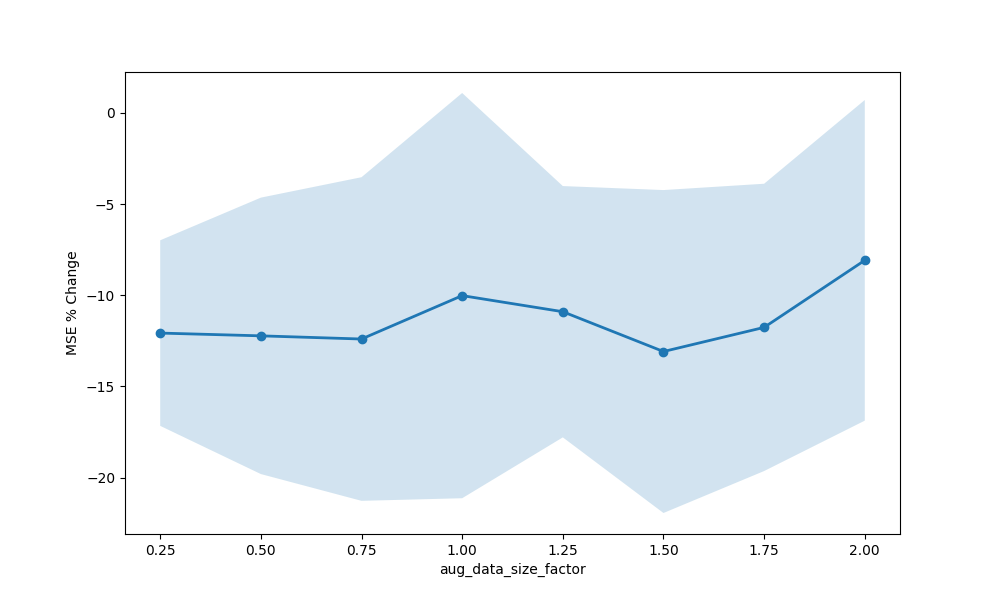}
    % \caption{\texttt{aug\_data\_size\_factor}}
  \end{subfigure}\hfill

  \begin{subfigure}[t]{0.62\linewidth}
    \centering
    \includegraphics[width=\linewidth]{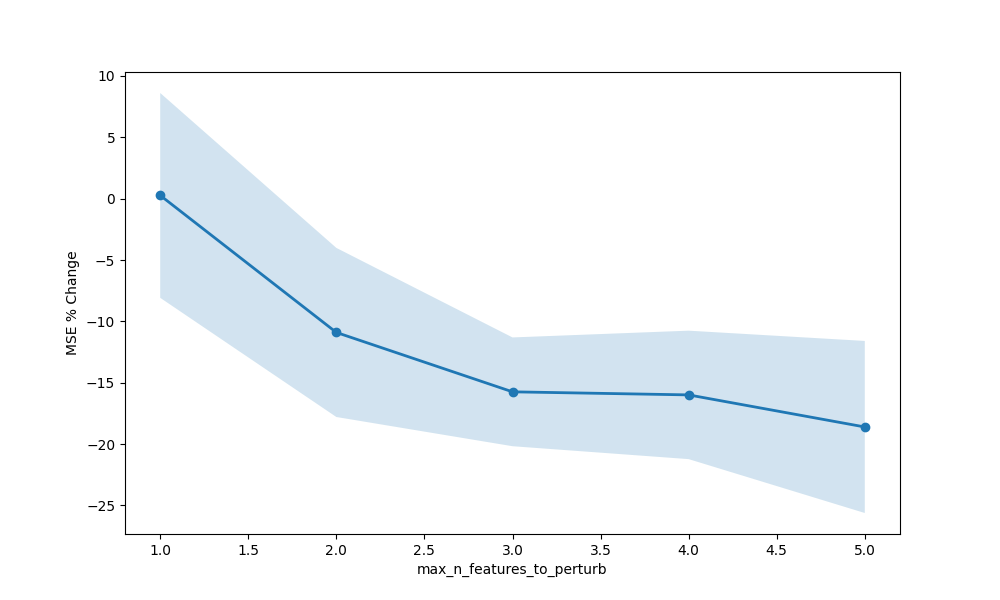}
    % \caption{\texttt{max\_n\_features\_to\_perturb}}
  \end{subfigure}\hfill

  \begin{subfigure}[t]{0.62\linewidth}
    \centering
    \includegraphics[width=\linewidth]{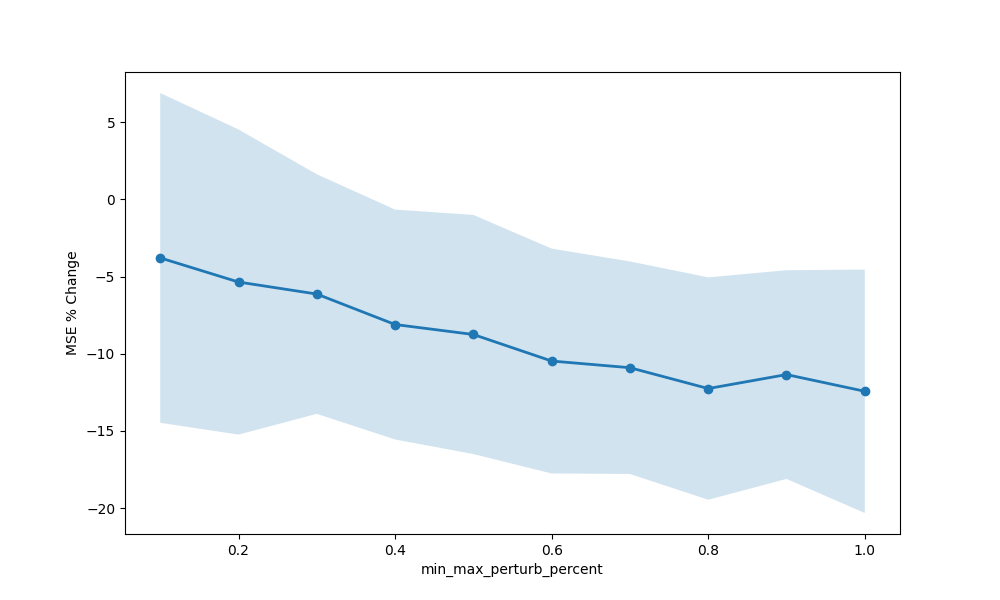}
    % \caption{\texttt{max\_perturb\_percent}}
  \end{subfigure}

  \caption{CRDA knob–sensitivity on the \textbf{XGB} baseline
           (House\-Price dataset).}
  \label{fig:knob-sens-xgb}
\end{figure}

\newpage

%-------------------------------------------------------------
\section{Causal Implications of Assumption~\ref{assump:cond}} \label{app:causal}
%-------------------------------------------------------------

This section gives causal analogues of our statistical independent assumptions and derives consequences for the causal relationship between the perturbed features $X_P$ and the target variable $Y$. 

\paragraph{Background: d-separation.}
We briefly recall the graphical notion underlying our assumptions \citep{pearl2009causality, peters2017elements}. A \emph{path} between two nodes is any sequence of distinct adjacent edges, regardless of their orientation. A node $W$ on a path is a \emph{collider} if both adjacent edges point into it ($\rightarrow W \leftarrow$). Given a conditioning set $S$, a path is \emph{blocked} if either (i) it contains a non-collider that is in $S$, or (ii) it contains a collider that is neither in $S$ nor has a descendant in $S$. Two nodes are \emph{d-separated given $S$} if every path between them is blocked, and \emph{d-connected given $S$} otherwise. The central property we use is that d-separation in the graph implies conditional independence in any distribution that factorizes according to it: if $A$ and $B$ are d-separated given $S$, then $A \perp B \mid S$ \citep{pearl2009causality}. The causal (d-separation) analogue of Assumption~\ref{assump:cond} is therefore that $Z$ is d-separated from $X_P$ given $X_R$.

We make the following assumption about the true data-generating causal graph $G$: (1) $Z$ is the only latent parent of $Y$, representing the common interpretation of $Z$ as summarizing unobserved causes influencing $Y$ and (2) $Z$ is d-separated from $X_P$ given $X_R$, which implies Condition~\ref{assump:cond}. In the following, write $\mathit{PA}(Y)$ for the parents of target variable $Y$.

\begin{lemma} \label{lemma:separation}
Let $G$ be any causal graph whose variables comprise $X_P,X_R,Z,Y$, and possibly other latent variables $U$, such that $Z$ is the only latent parent of $Y$ and $Z$ is d-separated from $X_P$ given $X_R$. Then every perturbed nonparent $X \in (X_P - \mathit{PA}(Y))$ is d-separated from $Y$ given $(\mathit{PA}(Y) \cap X_P) \cup X_R$.
\end{lemma}
\begin{proof} 
For contradiction, consider a path $X_0-X_1-\ldots-X_n-Y$ that d-connects $X_0 \in (X_P - \mathit{PA}(Y))$ with $Y$ given $(\mathit{PA}(Y) \cap X_P) \cup X_R$. We use  - to indicate a generic edge directed either way. 

Let $X_k \in X_P$ be the last perturbed variable on the path. That is, $k \geq 0$ is the largest index such that $X_k \in X_P$. Since subpaths of a d-connecting path are also d-connecting paths, the subpath $X_k-\ldots-X_n- Y$ d-connects $X_k$ to $Y$ given $(\mathit{PA}(Y) \cap X_P) \cup X_R$. Since all variables between $X_k$ and $X_n$ are in $X_R$, we have that (*) the subpath d-connects $X_k$ to $X_n$ given $X_R$ (only). 

Case 1: The subpath is of the form $X_k-\ldots-X_n \rightarrow Y$. 
%Then $n > k$ since $X_k$ is not a parent of $Y$. Also 
Then $X_n$ is not an observed variable, since otherwise conditioning on $(\mathit{PA}(Y) \cap X_P) \cup X_R$ would block the path. Since $Z$ is the only latent parent of $Y$, this means that $X_n = Z$. Therefore by the property (*), the subpath d-connects $X_k$ to $Z$ given $X_R$, which contradicts Assumption (2). 

Case 2: The subpath is of the form $X_k-\ldots-X_n \leftarrow Y$. Therefore $Y$ is not a collider on the path $X_k-\ldots-X_n \leftarrow Y \leftarrow Z$, and this path d-connects $Z$ to $X_k \in X_P$ given $X_R$, which again contradicts Assumption (2). 

Since both possible cases result in a contradiction, there is no path that d-connects a non-parent $X_k$ to the target variable $Y$ given the perturbed parents and $X_R$, which was to be proven.
\end{proof}

\begin{proposition} \label{prop:main-causal}
   Let $G$ be any causal graph whose variables comprise $X_P,X_R,Z,Y$, and possibly other latent variables $U$, such that $Z$ is the only latent parent of $Y$ and $Z$ is d-separated from $X_P$ given $X_R$.  Let $G'$ be the truncated graph in which the variables in $X_P$ have no parents. Then $P_{G}(Y|X) = P_{G'}(Y|X)$.  Therefore the CRDA procedure of Section~\ref{sec:method} computes the counterfactual $P(Y_{X=x'}|X=x,Y=y)$. 
\end{proposition}
\begin{proof}
Let $X_0 \in (X_P \cap \mathit{PA(Y)})$ be a perturbed parent of $Y$. The difference between the original graph $G$ and the truncated graph $G'$ is that the truncated graph removes backdoor paths that d-connect $X_0$ to $Y$ given $(\mathit{PA}(Y) \cap X_P) \cup X_R$. We show that there is no such backdoor path in the original graph. 

For contradiction, consider any backdoor path $X_0 \leftarrow X_1-\ldots-X_n-Y$ that d-connects $X_0 \in (X_P \cap \mathit{PA}(Y))$ with $Y$ given $(\mathit{PA}(Y) \cap X_P) \cup X_R$. As in the proof of Lemma~\ref{lemma:separation}, let $X_k \in X_P$ be the last perturbed variable on the path. As in the previous proof, we have that (*) the subpath d-connects $X_k$ to $X_n$ given $X_R$ (only). 

Case 1: The subpath is of the form $X_k-\ldots-X_n \rightarrow Y$. Since $Z$ is the only latent parent, there are only two possible cases.

Case 1a: $X_n = Z$. Then the subpath d-connects $X_k$ to $Z$ given $X_R$, which contradicts Assumption (2).
\\
Case 1b: $X_n$ is an observed variable. Then conditioning on $(\mathit{PA}(Y) \cap X_P) \cup X_R$ blocks a path ending in $X_n \rightarrow Y$ unless $X_n = X_0$. But then the path is of the form $X_0 \rightarrow Y$ and is not a backdoor path. 

Case 2: The subpath is of the form $X_k-\ldots-X_n \leftarrow Y$. Then as in Case 2 of Lemma~\ref{lemma:separation}, the path $X_k-\ldots-X_n \leftarrow Y \leftarrow Z$ d-connects $X_k$ to $Z$ given $X_R$, which contradicts Assumption 2.

Since there are no backdoor paths from perturbed parents in the non-truncated graph $G$, the conditional probabilities for perturbed parents are the same on both graphs:

\begin{align} \label{eq:cond-parents}
    P_G(Y|(\mathit{PA}(Y) \cap X_P) \cup X_R) = P_{G'}(Y|(\mathit{PA}(Y) \cap X_P) \cup X_R)
\end{align}

Since truncating only increases d-separation which implies independence, we conclude the following from Equation~\eqref{eq:cond-parents} and  Lemma~\ref{lemma:separation}:
\begin{align}
    P_G(Y|X_P,X_R) = P_G(Y|(\mathit{PA}(Y) \cap X_P) \cup X_R) = P_{G'}(Y|(\mathit{PA}(Y) \cap X_P) \cup X_R) = P_{G'}(Y|X_P,X_R)
  \label{eq:cond-final}
\end{align}
Equation~\eqref{eq:cond-final} shows that the conditional probabilities coincide in the original and truncated graphs, for every assignment of values to $X$, including $X = x'$.
Therefore the regressor $g$ fit on observational data estimates the correct interventional response. Under the additive-noise model, observing $(X=x, Y=y)$ further identifies the realized noise $Z = y - g(x)$, and reusing it while setting $X=x'$ yields $g(x') + z$; this is precisely CRDA's augmented label and equals the counterfactual outcome $Y_{X=x'}$, establishing that CRDA computes $P(Y_{X=x'} \mid X=x, Y=y)$.
\end{proof}

\paragraph{Remarks}
We note that Lemma \ref{lemma:separation} and Proposition \ref{prop:main-causal} above condition on $(\mathit{PA}(Y) \cap X_P) \cup X_R$ rather than on $X_R$ alone. The additional set $\mathit{PA}(Y) \cap X_P$ comprises the perturbed features that are themselves direct causes of $Y$; conditioning on them is necessary so that the remaining perturbed non-causes are blocked from $Y$, isolating the causal contribution of the perturbation. This is consistent with Assumption~\ref{assump:cond}, which constrains the residual $Z$ (the latent parent) and places no requirement on the observed parents of $Y$.

Our argument does not assume knowledge of the true causal graph $G$. The argument says that if we did know the true causal graph $G$, and used it to compute the counterfactual probability $P(Y_{X=x'}|X=x,Y=y)$, we would obtain the same result as estimating conditional probabilities $P(Y|X=x')$ after updating our posterior over the latent variable $Z$ given the observations $x,y$. 

The assumption that $Z$ is the only latent parent of $Y$ can be relaxed by allowing other latent parents but strengthening Assumption (2) to require that $X_P$ is d-separated from all latent parents of $Y$ given $X_R$. 

Proposition~\ref{prop:main-causal} is purely structural and does not depend on the particular form of the additive noise model. For example, it also applies to location-scale noise models~\citep{sun2023cause}.

%-------------------------------------------------------------
\section{Validation of Assumptions and Component Analysis}\label{app:validation_ablation}
%-------------------------------------------------------------

In this section, we provide a deeper analysis of the CRDA framework. First, we empirically validate the core residual independence assumption using Mutual Information. Second, we perform ablation studies to demonstrate the benefits of applying CRDA compared to simplified baselines.

%-------------------------------------------------------------
\subsection{Empirical Validation of Residual Independence}\label{app:divergence}
%-------------------------------------------------------------

A core theoretical assumption of CRDA (Assumption~\ref{assump:cond}) is that the residual noise $Z$ is conditionally independent of the features selected for perturbation ($X_P$), i.e., $P(Z | X_P, X_R) = P(Z | X_R)$. To validate this assumption and assess the effectiveness of our PC-algorithm/Correlation filter, we conducted an analysis measuring the Mutual Information (MI) between the residuals and the features. Mutual Information is an empirical estimator of the KL-Divergence $D_{KL}(P(Z, X) || P(Z)P(X))$; a value of zero indicates perfect independence.

We performed this evaluation across all 9 benchmark datasets using the XGBoost regressor over 15 random seeds. For each run, we calculated the MI (using the Kraskov KSG estimator \citep{kraskov2004estimating}) for the set of features \textit{Selected} ($X_P$) by CRDA versus the set of features \textit{Rejected} ($X_R$).

The results are presented in Table \ref{tab:mi_divergence_analysis}. We observe that for datasets showing stronger feature-residual dependence (e.g., Energy Efficiency, House Price), the features rejected by our filter display significantly higher Mutual Information with the residuals (up to $\approx3\times$ higher) than the selected features. This confirms that the filter effectively identifies and removes features that would violate the independence assumption. For other datasets (e.g., Wine Quality, Wind Power), the MI scores for both selected and rejected features are uniformly low, indicating that the residuals are naturally independent of the features in these domains, and the filter correctly permits a wider range of perturbations.

\begin{table}[h]
\centering
\caption{Evaluation of Feature-Residual Independence via Mutual Information (MI). We report the MI (in nats) between the model residuals $Z$ and the features $X$, comparing features \textbf{Selected} by CRDA vs. those \textbf{Rejected}. Results are averaged over 15 seeds with standard errors. The \textbf{Ratio} column highlights the effectiveness of the filter in reducing divergence (higher is better).}
\label{tab:mi_divergence_analysis}
% \resizebox{\textwidth}{!}{%
\begin{tabular}{l c c c}
\toprule
\textbf{Dataset} & \textbf{Selected Features ($X_P$)} & \textbf{Rejected Features ($X_R$)} & \textbf{Divergence Ratio} \\
& (Lower is better) & (Higher implies dependence) & ($MI_{Rej} / MI_{Sel}$) \\
\midrule
\textbf{House Price} & $\mathbf{0.0056 \pm 0.0020}$ & $0.0155 \pm 0.0023$ & $\mathbf{2.75\times}$ \\
\textbf{Energy Efficiency} & $\mathbf{0.0054 \pm 0.0012}$ & $0.0160 \pm 0.0023$ & $\mathbf{2.94\times}$ \\
\textbf{Parkinson's Monitoring} & $\mathbf{0.0054 \pm 0.0006}$ & $0.0103 \pm 0.0007$ & $\mathbf{1.92\times}$ \\
\textbf{Synthetic Regression} & $\mathbf{0.0073 \pm 0.0013}$ & $0.0136 \pm 0.0025$ & $\mathbf{1.85\times}$ \\
Concrete Strength & $0.0203 \pm 0.0024$ & $0.0320 \pm 0.0035$ & $1.58\times$ \\
CPU Performance & $0.0136 \pm 0.0010$ & $0.0144 \pm 0.0014$ & $1.06\times$ \\
Wine Quality & $0.0110 \pm 0.0009$ & $0.0136 \pm 0.0011$ & $1.23\times$ \\
Wind Power & $0.0065 \pm 0.0007$ & $0.0084 \pm 0.0007$ & $1.29\times$ \\
Satellite Image & $0.1031 \pm 0.0013$ & $0.1149 \pm 0.0009$ & $1.11\times$ \\
\bottomrule
\end{tabular}
% }
\end{table}

%-------------------------------------------------------------
\subsection{Ablation Studies}\label{app:simple_ablations}
%-------------------------------------------------------------

To verify that the independence assumption verified above translates to performance gains, we compare CRDA against two simplified ablation baselines:
\begin{itemize}
    \item \textbf{Global Perturbation:} All features are perturbed randomly ($X_P = X$), ignoring the PC-algorithm and correlation checks.
    \item \textbf{Label Invariance:} Features are perturbed, but the label is kept fixed ($y' = y$), rather than recalculating $y' = \hat{g}(x') + \hat{z}$.
\end{itemize}

Table~\ref{tab:ablation_results} presents the percentage change in MSE ($\Delta\%$) relative to the unaugmented base regressor across 3 representative datasets. CRDA consistently yields the largest error reduction. Notably, simple baselines often yield negligible improvements or even degrade performance (positive $\Delta\%$).

\begin{table}[h]
  \centering
  \caption{Ablation results on Synthetic Regression, Energy Efficiency, and Parkinson's Monitoring datasets. Values represent the percentage change in MSE ($\Delta\%$) relative to the unaugmented baseline (lower is better). Results are averaged over 5 seeds with standard errors.}
  \label{tab:ablation_results}
  \renewcommand{\arraystretch}{0.95}
  % \setlength{\tabcolsep}{5pt}
  % \resizebox{\textwidth}{!}{%
  \begin{tabular}{llccc}
    \toprule
    \multirow{2}{*}{\textbf{Dataset}} & \multirow{2}{*}{\textbf{Model}} & \multicolumn{3}{c}{\textbf{MSE $\Delta\%$ Change} ($\downarrow$)} \\
    \cmidrule(lr){3-5}
    & & Global Perturbation & Label Invariance & CRDA \\
    \midrule
    \multirow{2}{*}{\makecell[l]{Synthetic Regression}} 
      & MLP & $-16.12 \pm 4.30$ & $-12.44 \pm 32.43$ & $\mathbf{-38.94 \pm 4.02}$ \\
      & XGB & $+1.21 \pm 2.10$ & $-1.02 \pm 1.33$ & $\mathbf{-3.62 \pm 1.93}$ \\
    \midrule
    \multirow{2}{*}{\makecell[l]{Energy Efficiency}} 
      & MLP & $-14.50 \pm 3.86$ & $-2.65 \pm 2.47$ & $\mathbf{-38.84 \pm 5.99}$ \\
      & XGB & $-7.15 \pm 9.75$ & $-5.28 \pm 8.78$ & $\mathbf{-17.45 \pm 5.16}$ \\
    \midrule
    \multirow{2}{*}{\makecell[l]{Parkinson's Monitoring}} 
      & MLP & $-13.50 \pm 8.90$ & $+0.55 \pm 19.04$ & $\mathbf{-58.40 \pm 5.16}$ \\
      & XGB & $+0.36 \pm 1.57$ & $-3.09 \pm 4.35$ & $\mathbf{-7.82 \pm 2.47}$ \\
    \bottomrule
  \end{tabular}
% }
\end{table}

\newpage

%-------------------------------------------------------------
\section{Statistical Significance Tests}\label{app:wilcoxon}
%-------------------------------------------------------------
For every dataset $\times$ training–set–fraction of our main experiment we did a 10-fold cross validation comparison of \textsc{CRDA's} augmented MSE against the corresponding raw unaugmented MSE with a two--sided
Wilcoxon signed--rank test
($n_{\text{folds}}=10$, $n_{\text{seeds}}=15$ per cell).
The heat-maps in Figures~\ref{fig:heatmap-mlp} and ~\ref{fig:heatmap-xgb} visualize the outcome.

\begin{figure}[h]
  \centering
  \includegraphics[width=.8\linewidth]{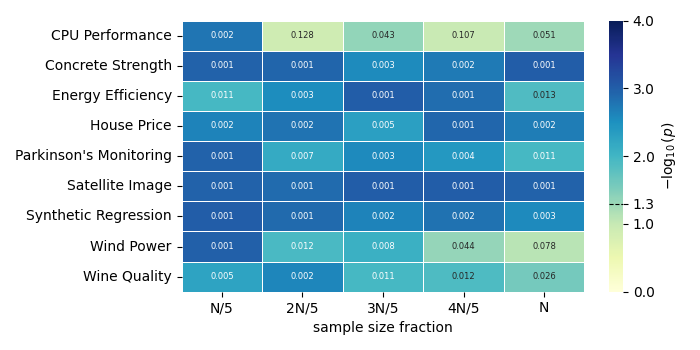}
  \caption{\textbf{MLP baseline.}  
           Colour encodes $-\!\log_{10}(p)$; numbers are the
           mean $p$ across 15 seeds.
           The dashed line on the colour-bar marks the
           $\alpha = 0.05$ threshold ($-\!\log_{10}p \approx 1.3$).}
  \label{fig:heatmap-mlp}
\end{figure}

\begin{figure}[!h]
  \centering
  \includegraphics[width=.8\linewidth]{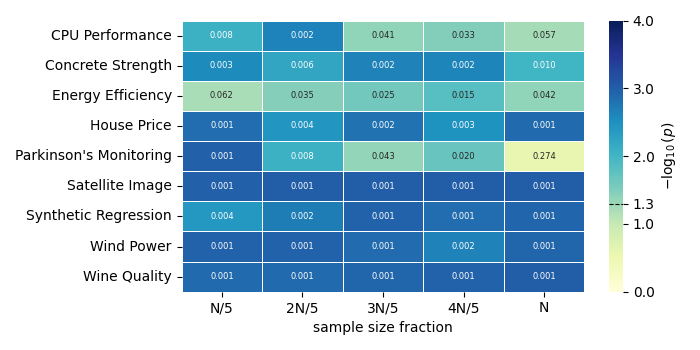}
  \caption{\textbf{XGB baseline.}  
           Same layout and colour scale as Figure ~\ref{fig:heatmap-mlp}.}
  \label{fig:heatmap-xgb}
\end{figure}

\textbf{Brief analysis.}
Across \emph{both} baselines the majority of cells are darker than the
$\alpha=0.05$ cut-off, indicating that
\textsc{CRDA} delivers a statistically significant reduction in test--MSE
for most dataset/size combinations.  
Significance is strongest for smaller training sets and occasionally
weakens as the full dataset is used
(e.g.\ \texttt{CPU Performance} and \texttt{Wind Power} for MLP,
\texttt{Parkinson's Monitoring} for XGB),
but even at $n$ the method remains significant in 7/9 datasets with at
least one baseline.  
These results support the robustness of the performance gains reported
in the main paper.

\newpage

%-------------------------------------------------------------
\section{Additional Baseline: CatBoost Analysis}
\label{app:catboost}
%-------------------------------------------------------------

To assess CRDA's robustness against stronger tree-based ensembles, we performed an additional evaluation using CatBoost \citep{prokhorenkova2018catboost}. CatBoost is often considered better than XGBoost due to its oblivious trees and robustness to overfitting, making it a challenging predictor to improve upon.

We evaluated performance at three fixed sample sizes ($N=\{300, 500, 700\}$) to observe behavior across different data availabilities.

Table~\ref{tab:catboost_results} presents the percentage change in MSE ($\Delta\%$). We observe three distinct behaviors:
\begin{itemize}
    \item \textbf{Consistent Gains:} On \textit{House Price} and \textit{Wind Power}, CRDA significantly reduces MSE across all sample sizes (peaking at -22.8\% for \textit{House Price}), demonstrating that CRDA behaves robustly for these tasks regardless of sample size.
    
    \item \textbf{Late-Stage Gains:} \textit{CPU Performance} requires a sufficient number of samples to model the residual. It shows no benefit at $N=300$ but improves substantially as data increases, reaching -13.0\% at $N=700$.
    \item \textbf{Sweet-Spot Behavior:} Datasets such as \textit{Parkinson's Monitoring}, \textit{Energy Efficiency}, and \textit{Synthetic Regression} exhibit a ``sweet spot" around $N=500$, where the augmentation provides the most benefit ($\approx$ 4-5\% reduction) before CatBoost potentially saturates the signal at larger sample sizes.
\end{itemize}

\begin{table}[h]
  \centering
  \caption{Average percentage change MSE ($\Delta\%$) for \textit{CatBoost} at fixed sample sizes across 15 seeds $\pm$ standard error.}
  \label{tab:catboost_results}
  \setlength{\tabcolsep}{5pt}
  \begin{tabular}{lccc}
    \toprule
    \textbf{Dataset} & \textbf{$N=300$} & \textbf{$N=500$} & \textbf{$N=700$} \\
    \midrule
    House Price & -22.80 $\pm$ 1.97 & -18.87 $\pm$ 1.69 & -14.11 $\pm$ 1.64 \\
    CPU Performance & 1.92 $\pm$ 2.99 & -7.34 $\pm$ 1.91 & -13.04 $\pm$ 2.96 \\
    Parkinson's Monitoring & -1.44 $\pm$ 1.99 & -5.13 $\pm$ 1.97 & -1.33 $\pm$ 1.49 \\
    Energy Efficiency & -1.65 $\pm$ 2.83 & -4.95 $\pm$ 3.32 & -1.89 $\pm$ 2.96 \\
    Synthetic Regression & -2.31 $\pm$ 1.86 & -4.01 $\pm$ 1.97 & -1.15 $\pm$ 2.16 \\
    Wind Power & -3.78 $\pm$ 1.23 & -2.54 $\pm$ 0.79 & -2.47 $\pm$ 0.95 \\
    Satellite Image & -0.95 $\pm$ 1.74 & -2.08 $\pm$ 1.34 & -1.95 $\pm$ 0.97 \\
    Wine Quality & -0.09 $\pm$ 1.28 & -1.62 $\pm$ 0.96 & -1.94 $\pm$ 0.68 \\
    Concrete Strength & -0.95 $\pm$ 1.85 & -0.69 $\pm$ 1.29 & 0.46 $\pm$ 1.26 \\
    \bottomrule
  \end{tabular}
\end{table}

\newpage

%-------------------------------------------------------------
\section{Linear Regression Base Predictor Study}\label{app:linreg}
%-------------------------------------------------------------
In order to test our method against weaker base predictors; where separate systematic signal cannot be cleanly separated from noise, possibly violating our assumptions; we selected linear regression. Using the same 15 seeds and data settings as the main experiment, we conducted this study to observe how CRDA behaves.

Table~\ref{tab:linreg_results} reports the averages and standard errors for baseline MSE, CRDA MSE, their percentage change ($\Delta$\,\%) as well as the $p-values$ from the Wilcoxon signed‑rank test for every dataset and sample size subset.

We see that CRDA’s filters \textit{rejected} every single fold. Recall that for the Wilcoxon signed‑rank test, if the p-value is above the 0.05 threshold, CRDA stops. We still report the $\Delta$\,\% if we had ignored the filter and observe how CRDA hurts here. CRDA therefore protects against weaker baselines, further illustrating how model-agnostic does not imply \textit{always helpful}.

\begin{table}[h]
  \scriptsize
  \centering
  \setlength\tabcolsep{3.5pt}
  \caption{Augmentation results for Linear Regression.  
           Cells are green when data augmentation
           was selected to proceed according to the Wilcoxon signed rank test and red otherwise. Lower is better for the $\Delta$ MSE \% change $\downarrow$.}
  \label{tab:linreg_results}
  \begin{tabular}{l c *{4}{c}}
    \toprule
    \multirow{2}{*}{\textbf{Dataset}}
      & \multirow{2}{*}{\textbf{Size}}
      & \multicolumn{4}{c}{\textbf{Linear Regression}} \\ 
    \cmidrule(lr){3-6}
      & & $\text{MSE}_{\mathrm{baseline}}$ & $\text{MSE}_{\mathrm{CRDA}}$ & $\Delta$\,\% $\downarrow$ & $\text{p-value}$
    \\
    \midrule

% ---------------- CPU Performance ------------------------------------
\multirow{5}{*}{CPU Performance} & 1638 & \bad{0.011094 $\pm$ 0.000901} & \bad{0.011066 $\pm$ 0.000871} & \bad{0.04 $\pm$ 0.60} & \bad{0.461 $\pm$ 0.036} \\
   & 3276 & \bad{0.010935 $\pm$ 0.000596} & \bad{0.011012 $\pm$ 0.000625} & \bad{0.58 $\pm$ 0.42} & \bad{0.506 $\pm$ 0.039} \\
   & 4914 & \bad{0.009789 $\pm$ 0.000305} & \bad{0.009784 $\pm$ 0.000306} & \bad{-0.05 $\pm$ 0.24} & \bad{0.452 $\pm$ 0.039} \\
   & 6552 & \bad{0.009834 $\pm$ 0.000372} & \bad{0.009887 $\pm$ 0.000374} & \bad{0.55 $\pm$ 0.19} & \bad{0.515 $\pm$ 0.030} \\
   & 8190 & \bad{0.009705 $\pm$ 0.000359} & \bad{0.009709 $\pm$ 0.000358} & \bad{0.05 $\pm$ 0.10} & \bad{0.456 $\pm$ 0.042} \\
% ---------------- Satellite Image ------------------------------------
\midrule
\multirow{5}{*}{Satellite Image} & 1287 & \bad{0.042119 $\pm$ 0.000676} & \bad{0.042185 $\pm$ 0.000681} & \bad{0.16 $\pm$ 0.15} & \bad{0.240 $\pm$ 0.038} \\
   & 2574 & \bad{0.041148 $\pm$ 0.000522} & \bad{0.041131 $\pm$ 0.000521} & \bad{-0.04 $\pm$ 0.07} & \bad{0.264 $\pm$ 0.033} \\
   & 3861 & \bad{0.040646 $\pm$ 0.000402} & \bad{0.040666 $\pm$ 0.000407} & \bad{0.05 $\pm$ 0.05} & \bad{0.275 $\pm$ 0.026} \\
   & 5148 & \bad{0.040154 $\pm$ 0.000315} & \bad{0.040183 $\pm$ 0.000306} & \bad{0.08 $\pm$ 0.05} & \bad{0.393 $\pm$ 0.038} \\
   & 6435 & \bad{0.040492 $\pm$ 0.000301} & \bad{0.040509 $\pm$ 0.000299} & \bad{0.04 $\pm$ 0.05} & \bad{0.347 $\pm$ 0.032} \\
% ---------------- Wind Power ------------------------------------
\midrule
\multirow{5}{*}{Wind Power} & 1314 & \bad{0.007335 $\pm$ 0.000271} & \bad{0.007339 $\pm$ 0.000270} & \bad{0.06 $\pm$ 0.11} & \bad{0.501 $\pm$ 0.044} \\
   & 2628 & \bad{0.006363 $\pm$ 0.000134} & \bad{0.006368 $\pm$ 0.000135} & \bad{0.08 $\pm$ 0.06} & \bad{0.468 $\pm$ 0.029} \\
   & 3942 & \bad{0.006580 $\pm$ 0.000102} & \bad{0.006583 $\pm$ 0.000102} & \bad{0.04 $\pm$ 0.04} & \bad{0.493 $\pm$ 0.030} \\
   & 5256 & \bad{0.006583 $\pm$ 0.000087} & \bad{0.006584 $\pm$ 0.000087} & \bad{0.00 $\pm$ 0.03} & \bad{0.498 $\pm$ 0.025} \\
   & 6570 & \bad{0.006175 $\pm$ 0.000052} & \bad{0.006175 $\pm$ 0.000051} & \bad{-0.00 $\pm$ 0.02} & \bad{0.528 $\pm$ 0.034} \\
% ---------------- Synthetic Regression ------------------------------------
\midrule
\multirow{5}{*}{Synthetic Regression} & 200 & \bad{0.023265 $\pm$ 0.001023} & \bad{0.023317 $\pm$ 0.001013} & \bad{0.29 $\pm$ 0.63} & \bad{0.306 $\pm$ 0.030} \\
   & 400 & \bad{0.022073 $\pm$ 0.000571} & \bad{0.022101 $\pm$ 0.000572} & \bad{0.13 $\pm$ 0.27} & \bad{0.365 $\pm$ 0.033} \\
   & 600 & \bad{0.021332 $\pm$ 0.000500} & \bad{0.021358 $\pm$ 0.000507} & \bad{0.12 $\pm$ 0.16} & \bad{0.385 $\pm$ 0.040} \\
   & 800 & \bad{0.015924 $\pm$ 0.000396} & \bad{0.015945 $\pm$ 0.000400} & \bad{0.12 $\pm$ 0.09} & \bad{0.381 $\pm$ 0.039} \\
   & 1000 & \bad{0.015908 $\pm$ 0.000378} & \bad{0.015911 $\pm$ 0.000373} & \bad{0.03 $\pm$ 0.13} & \bad{0.419 $\pm$ 0.032} \\
% ---------------- Concrete Strength ------------------------------------
\midrule
\multirow{5}{*}{Concrete Strength} & 201 & \bad{0.016621 $\pm$ 0.001084} & \bad{0.016592 $\pm$ 0.001080} & \bad{-0.15 $\pm$ 0.29} & \bad{0.362 $\pm$ 0.036} \\
   & 402 & \bad{0.017469 $\pm$ 0.000896} & \bad{0.017431 $\pm$ 0.000917} & \bad{-0.33 $\pm$ 0.32} & \bad{0.352 $\pm$ 0.032} \\
   & 603 & \bad{0.017323 $\pm$ 0.000489} & \bad{0.017336 $\pm$ 0.000507} & \bad{0.03 $\pm$ 0.26} & \bad{0.406 $\pm$ 0.029} \\
   & 804 & \bad{0.016711 $\pm$ 0.000571} & \bad{0.016726 $\pm$ 0.000582} & \bad{0.06 $\pm$ 0.16} & \bad{0.452 $\pm$ 0.029} \\
   & 1005 & \bad{0.015719 $\pm$ 0.000358} & \bad{0.015722 $\pm$ 0.000360} & \bad{0.01 $\pm$ 0.08} & \bad{0.477 $\pm$ 0.025} \\
% ---------------- Energy Efficiency ------------------------------------
\midrule
\multirow{5}{*}{Energy Efficiency} & 153 & \bad{0.003620 $\pm$ 0.000327} & \bad{0.003650 $\pm$ 0.000326} & \bad{1.02 $\pm$ 0.95} & \bad{0.413 $\pm$ 0.044} \\
   & 306 & \bad{0.002928 $\pm$ 0.000122} & \bad{0.002926 $\pm$ 0.000122} & \bad{-0.03 $\pm$ 0.37} & \bad{0.398 $\pm$ 0.039} \\
   & 459 & \bad{0.002872 $\pm$ 0.000124} & \bad{0.002873 $\pm$ 0.000123} & \bad{0.04 $\pm$ 0.23} & \bad{0.453 $\pm$ 0.042} \\
   & 612 & \bad{0.002615 $\pm$ 0.000086} & \bad{0.002621 $\pm$ 0.000084} & \bad{0.29 $\pm$ 0.31} & \bad{0.452 $\pm$ 0.036} \\
   & 765 & \bad{0.002667 $\pm$ 0.000056} & \bad{0.002673 $\pm$ 0.000059} & \bad{0.21 $\pm$ 0.20} & \bad{0.436 $\pm$ 0.026} \\
% ---------------- House Price ------------------------------------
\midrule
\multirow{5}{*}{House Price} & 200 & \bad{0.000103 $\pm$ 0.000006} & \bad{0.000103 $\pm$ 0.000006} & \bad{0.77 $\pm$ 0.54} & \bad{0.344 $\pm$ 0.030} \\
   & 400 & \bad{0.000100 $\pm$ 0.000004} & \bad{0.000101 $\pm$ 0.000004} & \bad{0.10 $\pm$ 0.19} & \bad{0.463 $\pm$ 0.039} \\
   & 600 & \bad{0.000100 $\pm$ 0.000004} & \bad{0.000100 $\pm$ 0.000004} & \bad{0.07 $\pm$ 0.13} & \bad{0.515 $\pm$ 0.024} \\
   & 800 & \bad{0.000104 $\pm$ 0.000003} & \bad{0.000104 $\pm$ 0.000003} & \bad{0.03 $\pm$ 0.09} & \bad{0.476 $\pm$ 0.047} \\
   & 1000 & \bad{0.000103 $\pm$ 0.000003} & \bad{0.000104 $\pm$ 0.000003} & \bad{0.22 $\pm$ 0.17} & \bad{0.445 $\pm$ 0.042} \\
% ---------------- Parkinson's Monitoring ------------------------------------
\midrule
\multirow{5}{*}{Parkinson's Monitoring} & 1175 & \bad{0.004750 $\pm$ 0.000112} & \bad{0.004754 $\pm$ 0.000112} & \bad{0.08 $\pm$ 0.10} & \bad{0.425 $\pm$ 0.043} \\
   & 2350 & \bad{0.004672 $\pm$ 0.000096} & \bad{0.004681 $\pm$ 0.000099} & \bad{0.18 $\pm$ 0.14} & \bad{0.333 $\pm$ 0.026} \\
   & 3525 & \bad{0.004651 $\pm$ 0.000096} & \bad{0.004650 $\pm$ 0.000095} & \bad{-0.02 $\pm$ 0.06} & \bad{0.355 $\pm$ 0.028} \\
   & 4700 & \bad{0.004618 $\pm$ 0.000077} & \bad{0.004620 $\pm$ 0.000076} & \bad{0.05 $\pm$ 0.05} & \bad{0.449 $\pm$ 0.030} \\
   & 5875 & \bad{0.004655 $\pm$ 0.000054} & \bad{0.004655 $\pm$ 0.000054} & \bad{0.01 $\pm$ 0.03} & \bad{0.429 $\pm$ 0.024} \\
% ---------------- Wine Quality ------------------------------------
\midrule
\multirow{5}{*}{Wine Quality} & 1063 & \bad{0.021526 $\pm$ 0.000658} & \bad{0.021544 $\pm$ 0.000670} & \bad{0.07 $\pm$ 0.24} & \bad{0.393 $\pm$ 0.026} \\
   & 2126 & \bad{0.015126 $\pm$ 0.000321} & \bad{0.015131 $\pm$ 0.000318} & \bad{0.04 $\pm$ 0.05} & \bad{0.452 $\pm$ 0.033} \\
   & 3189 & \bad{0.015448 $\pm$ 0.000206} & \bad{0.015461 $\pm$ 0.000204} & \bad{0.09 $\pm$ 0.05} & \bad{0.443 $\pm$ 0.027} \\
   & 4252 & \bad{0.014830 $\pm$ 0.000270} & \bad{0.014839 $\pm$ 0.000269} & \bad{0.06 $\pm$ 0.04} & \bad{0.487 $\pm$ 0.035} \\
   & 5315 & \bad{0.014894 $\pm$ 0.000168} & \bad{0.014896 $\pm$ 0.000170} & \bad{0.01 $\pm$ 0.03} & \bad{0.505 $\pm$ 0.024} \\

    \bottomrule
  \end{tabular}
\end{table}
%HOSS

\newpage

%-------------------------------------------------------------
\section{Complete Per-Dataset Scores}\label{app:perseed}
%-------------------------------------------------------------
Table~\ref{tab:all_results} reports the baseline MSE, CRDA MSE and their percentage change ($\Delta$\,\%) for every dataset and sample size subset. It is a more comprehensive version of Table 1 in the main paper.  These results are the averages across 15 different seed runs and so we include their standard errors.\footnote{Per-seed results are available in the code repository at
\texttt{experiments/\{dataset\}/\{model\}/interim\_results}}

\begin{table}[h]
  \tiny
  \centering
  \setlength\tabcolsep{3.5pt}
  \caption{Complete results with standard errors for XGB and MLP across 15 seeds for each of the 9 datasets. Lower is better $\downarrow$.}
  \label{tab:all_results}
  \begin{tabular}{l c *{3}{c} *{3}{c}}
    \toprule
    \multirow{2}{*}{\textbf{Dataset}}
      & \multirow{2}{*}{\textbf{Sample Size}}
      & \multicolumn{3}{c}{\textbf{XGB} $\downarrow$}
      & \multicolumn{3}{c}{\textbf{MLP} $\downarrow$} \\ 
    \cmidrule(lr){3-5}\cmidrule(lr){6-8}
      & & $\text{MSE}_{\mathrm{baseline}}$ & $\text{MSE}_{\mathrm{CRDA}}$ & $\Delta$\,\% &
          $\text{MSE}_{\mathrm{baseline}}$ & $\text{MSE}_{\mathrm{CRDA}}$ & $\Delta$\,\% \\
    \midrule

% ---------------- CPU Performance ------------------------------------
\midrule
\multirow{5}{*}{CPU Performance} & 1638 & 0.00097 $\pm$ 0.00005 & 0.00089 $\pm$ 0.00004 & -7.0 $\pm$ 2.8 & 0.00112 $\pm$ 0.00007 & 0.00087 $\pm$ 0.00003 & -20.2 $\pm$ 3.3 \\
        & 3276 & 0.00088 $\pm$ 0.00003 & 0.00079 $\pm$ 0.00002 & -9.5 $\pm$ 2.3 & 0.00100 $\pm$ 0.00002 & 0.00085 $\pm$ 0.00002 & -14.0 $\pm$ 1.6 \\
        & 4914 & 0.00077 $\pm$ 0.00002 & 0.00072 $\pm$ 0.00001 & -6.2 $\pm$ 1.3 & 0.00093 $\pm$ 0.00002 & 0.00082 $\pm$ 0.00001 & -11.3 $\pm$ 1.6 \\
        & 6552 & 0.00073 $\pm$ 0.00002 & 0.00069 $\pm$ 0.00001 & -4.1 $\pm$ 1.7 & 0.00090 $\pm$ 0.00003 & 0.00079 $\pm$ 0.00001 & -10.5 $\pm$ 2.2 \\
        & 8190 & 0.00074 $\pm$ 0.00002 & 0.00070 $\pm$ 0.00001 & -5.2 $\pm$ 2.0 & 0.00087 $\pm$ 0.00001 & 0.00078 $\pm$ 0.00001 & -10.2 $\pm$ 0.8 \\

% ---------------- Satellite Image ------------------------------------
\midrule
\multirow{5}{*}{Satellite Image} & 1287 & 0.01778 $\pm$ 0.00047 & 0.01697 $\pm$ 0.00053 & -4.5 $\pm$ 1.5 & 0.02031 $\pm$ 0.00104 & 0.01629 $\pm$ 0.00059 & -18.4 $\pm$ 3.2 \\
        & 2574 & 0.01636 $\pm$ 0.00037 & 0.01576 $\pm$ 0.00041 & -3.7 $\pm$ 0.9 & 0.01747 $\pm$ 0.00039 & 0.01455 $\pm$ 0.00041 & -16.7 $\pm$ 1.5 \\
        & 3861 & 0.01460 $\pm$ 0.00035 & 0.01390 $\pm$ 0.00035 & -4.8 $\pm$ 0.8 & 0.01585 $\pm$ 0.00054 & 0.01211 $\pm$ 0.00037 & -23.1 $\pm$ 1.8 \\
        & 5148 & 0.01366 $\pm$ 0.00033 & 0.01300 $\pm$ 0.00031 & -4.7 $\pm$ 1.1 & 0.01415 $\pm$ 0.00045 & 0.01076 $\pm$ 0.00030 & -23.7 $\pm$ 1.3 \\
        & 6435 & 0.01254 $\pm$ 0.00030 & 0.01186 $\pm$ 0.00027 & -5.3 $\pm$ 0.8 & 0.01232 $\pm$ 0.00035 & 0.00989 $\pm$ 0.00029 & -19.7 $\pm$ 1.1 \\

% ---------------- Wind Power ------------------------------------
\midrule
\multirow{5}{*}{Wind Power} & 1314 & 0.00742 $\pm$ 0.00029 & 0.00721 $\pm$ 0.00029 & -2.8 $\pm$ 1.2 & 0.00752 $\pm$ 0.00025 & 0.00697 $\pm$ 0.00025 & -7.2 $\pm$ 1.6 \\
        & 2628 & 0.00602 $\pm$ 0.00012 & 0.00603 $\pm$ 0.00013 & 0.2 $\pm$ 0.6 & 0.00621 $\pm$ 0.00017 & 0.00562 $\pm$ 0.00011 & -9.2 $\pm$ 1.5 \\
        & 3942 & 0.00586 $\pm$ 0.00008 & 0.00578 $\pm$ 0.00008 & -1.3 $\pm$ 0.4 & 0.00593 $\pm$ 0.00008 & 0.00539 $\pm$ 0.00008 & -9.0 $\pm$ 0.7 \\
        & 5256 & 0.00570 $\pm$ 0.00006 & 0.00562 $\pm$ 0.00007 & -1.4 $\pm$ 0.4 & 0.00567 $\pm$ 0.00009 & 0.00533 $\pm$ 0.00009 & -6.2 $\pm$ 0.4 \\
        & 6570 & 0.00528 $\pm$ 0.00005 & 0.00522 $\pm$ 0.00005 & -1.1 $\pm$ 0.3 & 0.00530 $\pm$ 0.00004 & 0.00500 $\pm$ 0.00004 & -5.6 $\pm$ 0.5 \\

% ---------------- Synthetic Regression ------------------------------------
\midrule
\multirow{5}{*}{Synthetic Regression} & 200 & 0.00652 $\pm$ 0.00044 & 0.00564 $\pm$ 0.00032 & -12.0 $\pm$ 3.8 & 0.01993 $\pm$ 0.00178 & 0.01387 $\pm$ 0.00163 & -28.8 $\pm$ 6.5 \\
        & 400 & 0.00327 $\pm$ 0.00027 & 0.00312 $\pm$ 0.00023 & -3.2 $\pm$ 2.5 & 0.00610 $\pm$ 0.00037 & 0.00384 $\pm$ 0.00032 & -36.9 $\pm$ 3.1 \\
        & 600 & 0.00264 $\pm$ 0.00009 & 0.00242 $\pm$ 0.00007 & -7.9 $\pm$ 2.2 & 0.00321 $\pm$ 0.00027 & 0.00228 $\pm$ 0.00019 & -27.9 $\pm$ 3.1 \\
        & 800 & 0.00165 $\pm$ 0.00008 & 0.00161 $\pm$ 0.00009 & -2.2 $\pm$ 2.1 & 0.00223 $\pm$ 0.00017 & 0.00140 $\pm$ 0.00008 & -34.1 $\pm$ 4.1 \\
        & 1000 & 0.00152 $\pm$ 0.00005 & 0.00145 $\pm$ 0.00006 & -4.6 $\pm$ 2.9 & 0.00220 $\pm$ 0.00014 & 0.00123 $\pm$ 0.00007 & -42.3 $\pm$ 3.2 \\

% ---------------- Concrete Strength ------------------------------------
\midrule
\multirow{5}{*}{Concrete Strength} & 201 & 0.00777 $\pm$ 0.00071 & 0.00701 $\pm$ 0.00065 & -8.0 $\pm$ 3.8 & 0.01033 $\pm$ 0.00106 & 0.00793 $\pm$ 0.00052 & -17.8 $\pm$ 5.9 \\
        & 402 & 0.00493 $\pm$ 0.00036 & 0.00453 $\pm$ 0.00038 & -8.4 $\pm$ 2.8 & 0.00635 $\pm$ 0.00055 & 0.00496 $\pm$ 0.00038 & -19.8 $\pm$ 2.9 \\
        & 603 & 0.00473 $\pm$ 0.00025 & 0.00427 $\pm$ 0.00025 & -9.7 $\pm$ 2.2 & 0.00602 $\pm$ 0.00015 & 0.00494 $\pm$ 0.00014 & -17.6 $\pm$ 2.5 \\
        & 804 & 0.00365 $\pm$ 0.00017 & 0.00307 $\pm$ 0.00015 & -15.7 $\pm$ 1.9 & 0.00497 $\pm$ 0.00026 & 0.00361 $\pm$ 0.00013 & -24.8 $\pm$ 4.2 \\
        & 1005 & 0.00290 $\pm$ 0.00010 & 0.00256 $\pm$ 0.00013 & -12.2 $\pm$ 2.1 & 0.00422 $\pm$ 0.00025 & 0.00306 $\pm$ 0.00017 & -26.9 $\pm$ 1.8 \\

% ---------------- Energy Efficiency ------------------------------------
\midrule
\multirow{5}{*}{Energy Efficiency} & 153 & 0.00399 $\pm$ 0.00050 & 0.00344 $\pm$ 0.00052 & -13.3 $\pm$ 8.0 & 0.00583 $\pm$ 0.00049 & 0.00426 $\pm$ 0.00048 & -25.1 $\pm$ 7.0 \\
        & 306 & 0.00233 $\pm$ 0.00015 & 0.00206 $\pm$ 0.00015 & -12.2 $\pm$ 3.3 & 0.00321 $\pm$ 0.00015 & 0.00233 $\pm$ 0.00021 & -28.1 $\pm$ 5.0 \\
        & 459 & 0.00165 $\pm$ 0.00012 & 0.00143 $\pm$ 0.00012 & -10.5 $\pm$ 6.1 & 0.00188 $\pm$ 0.00016 & 0.00106 $\pm$ 0.00013 & -43.0 $\pm$ 4.8 \\
        & 612 & 0.00128 $\pm$ 0.00008 & 0.00100 $\pm$ 0.00006 & -19.3 $\pm$ 5.5 & 0.00091 $\pm$ 0.00008 & 0.00052 $\pm$ 0.00006 & -40.7 $\pm$ 4.0 \\
        & 765 & 0.00097 $\pm$ 0.00006 & 0.00076 $\pm$ 0.00007 & -21.0 $\pm$ 4.5 & 0.00053 $\pm$ 0.00008 & 0.00035 $\pm$ 0.00003 & -28.3 $\pm$ 4.4 \\

% ---------------- House Price ------------------------------------
\midrule
\multirow{5}{*}{House Price} & 200 & 0.00079 $\pm$ 0.00009 & 0.00064 $\pm$ 0.00005 & -14.2 $\pm$ 4.9 & 0.00102 $\pm$ 0.00012 & 0.00057 $\pm$ 0.00007 & -40.6 $\pm$ 5.0 \\
        & 400 & 0.00033 $\pm$ 0.00002 & 0.00031 $\pm$ 0.00002 & -5.4 $\pm$ 2.3 & 0.00041 $\pm$ 0.00003 & 0.00025 $\pm$ 0.00001 & -37.0 $\pm$ 3.8 \\
        & 600 & 0.00027 $\pm$ 0.00002 & 0.00026 $\pm$ 0.00002 & -4.9 $\pm$ 2.8 & 0.00029 $\pm$ 0.00002 & 0.00020 $\pm$ 0.00002 & -30.1 $\pm$ 3.9 \\
        & 800 & 0.00024 $\pm$ 0.00001 & 0.00022 $\pm$ 0.00001 & -9.9 $\pm$ 2.1 & 0.00023 $\pm$ 0.00001 & 0.00016 $\pm$ 0.00001 & -30.3 $\pm$ 4.2 \\
        & 1000 & 0.00020 $\pm$ 0.00001 & 0.00018 $\pm$ 0.00001 & -6.5 $\pm$ 1.9 & 0.00019 $\pm$ 0.00001 & 0.00014 $\pm$ 0.00001 & -27.0 $\pm$ 2.6 \\

% ---------------- Parkinson's Monitoring ------------------------------------
\midrule
\multirow{5}{*}{Parkinson's Monitoring} & 1175 & 0.00079 $\pm$ 0.00003 & 0.00072 $\pm$ 0.00003 & -8.4 $\pm$ 2.5 & 0.00165 $\pm$ 0.00012 & 0.00101 $\pm$ 0.00006 & -36.2 $\pm$ 4.0 \\
        & 2350 & 0.00034 $\pm$ 0.00002 & 0.00032 $\pm$ 0.00001 & -6.6 $\pm$ 2.9 & 0.00080 $\pm$ 0.00005 & 0.00054 $\pm$ 0.00003 & -31.8 $\pm$ 2.6 \\
        & 3525 & 0.00021 $\pm$ 0.00001 & 0.00020 $\pm$ 0.00001 & -2.8 $\pm$ 3.5 & 0.00048 $\pm$ 0.00003 & 0.00030 $\pm$ 0.00002 & -36.6 $\pm$ 4.1 \\
        & 4700 & 0.00015 $\pm$ 0.00001 & 0.00014 $\pm$ 0.00001 & -6.3 $\pm$ 2.5 & 0.00042 $\pm$ 0.00003 & 0.00021 $\pm$ 0.00001 & -46.4 $\pm$ 4.2 \\
        & 5875 & 0.00011 $\pm$ 0.00001 & 0.00011 $\pm$ 0.00001 & 1.7 $\pm$ 3.9 & 0.00026 $\pm$ 0.00002 & 0.00013 $\pm$ 0.00001 & -47.2 $\pm$ 4.8 \\

% ---------------- Wine Quality ------------------------------------
\midrule
\multirow{5}{*}{Wine Quality} & 1063 & 0.02057 $\pm$ 0.00058 & 0.02062 $\pm$ 0.00056 & 0.3 $\pm$ 0.8 & 0.02291 $\pm$ 0.00091 & 0.02284 $\pm$ 0.00133 & -0.3 $\pm$ 3.4 \\
        & 2126 & 0.01416 $\pm$ 0.00030 & 0.01429 $\pm$ 0.00030 & 1.0 $\pm$ 0.8 & 0.01539 $\pm$ 0.00027 & 0.01458 $\pm$ 0.00034 & -5.2 $\pm$ 1.6 \\
        & 3189 & 0.01391 $\pm$ 0.00020 & 0.01386 $\pm$ 0.00017 & -0.3 $\pm$ 0.5 & 0.01478 $\pm$ 0.00024 & 0.01423 $\pm$ 0.00025 & -3.6 $\pm$ 1.6 \\
        & 4252 & 0.01332 $\pm$ 0.00025 & 0.01324 $\pm$ 0.00027 & -0.6 $\pm$ 0.5 & 0.01386 $\pm$ 0.00028 & 0.01323 $\pm$ 0.00025 & -4.4 $\pm$ 0.9 \\
        & 5315 & 0.01332 $\pm$ 0.00013 & 0.01318 $\pm$ 0.00015 & -1.1 $\pm$ 0.3 & 0.01397 $\pm$ 0.00017 & 0.01328 $\pm$ 0.00020 & -5.0 $\pm$ 0.6 \\

    \bottomrule
  \end{tabular}
\end{table}

%-------------------------------------------------------------
\section{Compute Budget and Carbon Footprint}\label{app:compute}
%-------------------------------------------------------------
\begin{itemize}
  \item \textbf{Hardware.} AWS \texttt{c7i.24xlarge} (96 vCPU, 192 GB RAM, Xeon
        Platinum 8480C, \(\approx\)0.59 kW active draw).\footnote{Power estimate
        from Intel C7i workload proof sheet.}
  \item \textbf{Runtime.} 13.562 h total  
        (9.103 h MLP baseline, 4.459 h XGB baseline).
  \item \textbf{Energy.} \(13.562 \times 0.59 \approx\ 8.0\) kWh.
  \item \textbf{CO\(_2\)-eq.} Local grid intensity  
        \(34.5\ \text{g CO}_2/\text{kWh}\)
        \(\Rightarrow 8.0 \times 0.0345 \approx\ \text{0.28 kg CO}_2\).
\end{itemize}

\end{document}